\documentclass[11pt, letterpaper]{article}
\usepackage[top=1in,bottom=1in,left=1in,right=1in]{geometry}
\usepackage{microtype}
\usepackage{graphicx}
\usepackage{booktabs} 
\usepackage[utf8]{inputenc} 
\usepackage[T1]{fontenc}    
\usepackage{hyperref}       
\usepackage{url}            
\usepackage{booktabs}       
\usepackage{amsfonts}       
\usepackage{nicefrac}       
\usepackage{microtype}      
\usepackage{xcolor}         

\usepackage{hyperref}

\title{Meta Sparse Principal Component Analysis}

\usepackage{amsmath}
\usepackage{amssymb}
\usepackage{mathtools}
\usepackage{amsthm}
\usepackage{caption}
\usepackage{subfig}
\usepackage[capitalize,noabbrev]{cleveref}

\theoremstyle{plain}
\newtheorem{theorem}{Theorem}[section]
\newtheorem{proposition}[theorem]{Proposition}
\newtheorem{lemma}[theorem]{Lemma}

\theoremstyle{definition}
\newtheorem{definition}[theorem]{Definition}
\newtheorem{assumption}[theorem]{Assumption}
\theoremstyle{remark}
\newtheorem{remark}[theorem]{Remark}

\newcommand{\lp}{\left(}
\newcommand{\rp}{\right)}
\newcommand{\lc}{\left\{}
\newcommand{\rc}{\right\}}
\newcommand{\lb}{\left[}
\newcommand{\rb}{\right]}

\newcommand{\Cov}{\mathrm{Cov}}

\newcommand{\constant}{\mathbb{C}}

\newcommand{\tr}{\mathrm{Trace}}

\newcommand{\Ncal}{\mathcal{N}}

\newcommand{\Acal}{\mathcal{A}}

\newcommand{\Fcal}{\mathcal{F}}

\newcommand{\pihat}{\hat{\Pi}}

\newcommand{\real}{\mathbb{R}}

\newcommand{\inner}[2]{\langle #1, #2 \rangle}

\newcommand{\bbb}{\mathbb{B}}
\newcommand\numberthis{\addtocounter{equation}{1}\tag{\theequation}}
\newcommand{\ninf}[1]{\left\| #1 \right\|_{\infty,\infty}}
\newcommand{\Ibb}{\mathbb{I}}
\newcommand{\Dbb}{\mathbb{D}}

\newcommand{\Gcal}{\mathcal{G}}
\newcommand{\prob}{\mathbb{P}}
\newcommand{\expec}{\mathbb{E}}
\newcommand{\initialD}{D_0}

\usepackage[textsize=tiny]{todonotes}

\author{ Imon Banerjee\footnote{Departments of Statistics} \ and Jean Honorio\footnote{Departments of Computer Science}\\
  Purdue University\\
  West-Lafayette, IN 47906 \\
  \texttt{ibanerj,jhonorio@purdue.edu} \\}

\date{}

\begin{document}
\maketitle

\vskip 0.3in

\begin{abstract}
   We study the meta-learning for support (i.e. the set of non-zero entries) recovery in high-dimensional Principal Component  Analysis. We reduce the sufficient sample complexity in a novel task with the information that is learned from auxiliary tasks. We assume each task to be a different random Principal Component (PC) matrix with a possibly different support and that the support union of the PC matrices is small. We then pool the data from all the tasks to execute an improper estimation of a single PC matrix by maximising the $l_1$-regularised predictive covariance to establish that with high probability the true support union can be recovered provided a sufficient number of tasks $m$ and a sufficient number of samples $ O\lp\frac{\log(p)}{m}\rp$ for each task, for $p$-dimensional vectors. Then, for a novel task, we prove that the maximisation of the $l_1$-regularised predictive covariance with the additional constraint that the support is a subset of the estimated support union could reduce the sufficient sample complexity of successful support recovery to $O(\log |J|)$, where $J$ is the support union recovered from the auxiliary tasks. Typically, $|J|$ would be much less than $p$ for sparse matrices. Finally, we demonstrate the validity of our experiments through numerical simulations.
\end{abstract}

\section{Introduction}~\label{sec:introduction}
Principal component analysis (PCA) is an important problem in high-dimensional statistical learning with several applications in information theory \cite{deshpande2014it}, image classification \cite{sun2019image}, brain computer interface \cite{lin2008development}, speech recognition \cite{zheng2015attribute} among others. PCA seeks a linear transformation that maximises the predictive covariance of the data, while depending on a small number of variables. The learner faces several challenges in PCA, such as high dimension or severe heterogeneity of the data. To be precise, the dimension of the data, $p$, could be much higher than the sample size $n$, making estimation infeasible. Or there could be limited samples from the distribution of interest but a large amount of samples from multiple multivariate distributions with different principal component (PC) matrices.

\paragraph{Practical Applications.} {Even amongst the varied applications of PCA, there exists an important niche concerned with estimating multiple PC matrices simultaneously. Here we present some immediate examples. The recent paper by \cite{persichetti2021data}  produces a data driven approach to map the temporal lobes of brain where they calculated $41$ separate covariance matrices each of which can be analysed using PCA. \cite{zhang2008global} calculates the PC matrices of solar insolation over the Pacific region for $264$ months from July 1983 to June 2005, simultaneously calculating $264$ PC matrices. These matrices are often sparse, having only one or two significant eigenvalues. Therefore, it is obvious that there are far-reaching consequences of being able to simultaneously estimate multiple sparse PC matrices.} 

\paragraph{Motivation. }Theoretical results for consistency, rates of convergence, minimax risk bounds for estimating eigenvectors and principal subspaces have been considered in various prior works \cite{zou2006sparse,johnstone2009consistency,jenatton2010structured,ma2013sparse,lounici2013sparse,park2019sparse}. It is well documented that the classical PCA has major practical and theoretical drawbacks when it is applied to high-dimensional data. The estimates PCA produces in high-dimensional situations are often inconsistent \cite{paul2007asymptotics,nadler2008finite,johnstone2009consistency,lei2015sparsistency}. The loadings of PC matrices are typically non-zero. This often makes it difficult to interpret the PC matrix and identify important variables. For this reason, estimating the PC matrix to recover its support set, which is the set of non-zero entries, is a common strategy of structure learning. An estimate of the PC matrix will be called sign consistent if it has the same support and sign of entries with respect to the true PC matrix.

For the first challenge, we assume the PC matrices are sparse and consider a general class of distributions (sub-Gaussian). To address the challenge of heterogeneity, we consider a meta-learning learning problem where the learner treats each different distribution as a task with a related PC matrix and solves every task simultaneously. We suppose there are $m$ ``auxiliary" tasks and $n$ samples with dimension $p$ per task. When there is only one task ($m=1$), \cite{lei2015sparsistency} proved that, under general assumptions, $n\in O(\log p)$ is sufficient for the sign consistency of the PC matrix under $l_1$ regularisation. Similarly, when $m=1$, using a mixed convex norm, \cite{qi2013sparse} proved the consistency of PCA. However, neither of them performs support recovery while simultaneously accounting for the heterogeneity of data.

In this paper, we solve the heterogeneity challenge with meta-learning. A single task denotes the recovery of the $k$-dimensional principal subspace of a given covariance matrix. We recover the support of the PC matrix in a novel task with the information learned from other auxiliary tasks. We also use improper estimation in our meta-learning method to obtain novel theoretical guarantees for support recovery. Instead of estimating every PC matrix in the auxiliary tasks, which might require more samples than feasible, we pool all the samples from the auxiliary tasks together to estimate a single common PC matrix and recover the ``support union'' (termed $J$). This allows us to provide statistical guarantees for the sign consistency when $m\in O(\log p)$. Moreover, given the recovered support union, we also provide statistical guarantees for estimating any extra unimportant variables in a novel task when the number of samples in the novel task is $O(\log |J|)$. This implies substantial reductions in the number of samples required for the support recovery of the novel task.

To the best of our knowledge, this is the first work concerning the sign-consistency of sparse PC matrices under a meta-learning setup. We were able to introduce randomness in the PC matrices of different tasks. Our theoretical results hold for a wide class of distributions of the PC matrices under some conditions, which implies broad application scenarios of our method. The use of improper estimation in our method is innovative for the problem of support recovery of high-dimensional PC matrices. Thereby, our work fills the gap between theory and methodology of meta-learning in high-dimensional PCA.
\begin{table}[!ht]
\caption{Notations used in the paper}
    \begin{footnotesize}
    \begin{tabular}{llll}
    \hline
	   Notation & Description & Notation & Description \\
	\hline
		$[a]$ & The first $a$ integers, i.e. $\lc 1,2,\dots,a\rc$. &
	$|J|$ & Cardinality of a set $J$.\\
	$A_{i,j}$ & The  $i,j^{th}$ cell of matrix $A$. & $U_i$ & The $i^{th}$ coordinate of a vector $u$.\\
	$supp(A)$ & Support of $diag(A)$, i.e. $\lc i: A_{i,i}\neq 0\rc$. & $A_{*,i}$ & The $i^{th}$ column of matrix $A$.\\
	$A_{i,*}$ & The $i^{th}$ row of $A$. & $sign(x)$ & The sign of $x, x\in\mathbb{R}$.\\
    $\Sigma$ & The true covariance $A$. & $\Sigma^{(i)}$ &  Covariance of auxiliary task $i$.\\
    $\lambda_k(A)$ & The $k^{th}$ largest eigenvalue of $A$ & $\|u\|_q$ & $l_q$-norm of the vector $u\in \real^p$\\
    $\|A\|_{q_1,q_2}$ & $\|\|A_{1,\infty}\|_{q_1},\dots,\|A_{p,*}\|_{q_1}\|_{q_2}$. & $\|A\|_F$ & The Frobenius norm of a matrix A \\
    $\inner{A}{B}$ & $\tr(A^TB)$. & $diag(A)$ & Diagonal of matrix $A$\\
    $\Dbb(A)$ & The diagonal matrix with elements $A_{i,i}$ & $\Ibb_p$ & The $p$-dimensional identity matrix
   \\
   $A_{I,J}$ & Submatrix of $A$ with rows $I$ and columns $J$&\\
    \hline
\end{tabular}
\end{footnotesize}
\label{tab:notations}
\end{table}
\paragraph{Contributions. } This paper has the following four contributions. In section 3, we {propose a meta-learning approach} by introducing multiple auxiliary learning tasks for support recovery of high-dimensional PC matrices with improper estimation. We\text{ define a generative model} for population covariance matrices that is amenable for analysis. In section 4, we prove that  $O\lp\frac{\log p}{m}\rp$ is \text{a sufficient sample complexity for support union recovery} for $p$-dimensional multivariate sub-Gaussian random vectors and $m$ auxiliary tasks with support union $J$. 
In section 5, we prove that given a recovered support union from the auxiliary tasks, \text{we can reduce the dimension of the novel} task from $p\times p$ to $|J|\times |J|$. 
We then use this fact to establish a sample complexity of $O(\log(|J|))$ for \text{support recovery of the novel task}. This provides the theoretical basis for introducing more tasks for meta-learning in support recovery of PC matrices. 
Lastly, we conduct synthetic experiments to validate our theory. We calculate the support union recovery rates of our meta-learning approach for different sizes of samples and tasks. For a fixed task size $m$, our approach achieves high support union recovery rates when the sample size per task has the order $O(\frac{\log p}{m})$. For a fixed sample size per task, our method performs the best when the task size $m$ is large.

\section{Preliminaries}\label{sec:preliminaries}
In this section, we introduce the definitions of the concepts important to our meta-learning problem. A summary of notations used in the paper is illustrated in \Cref{tab:notations}.

\paragraph{Multivariate sub-Gaussian Distribution. }

This paper is concerned with meta-learning of multiple principal component matrices, each of which are generated randomly. To make inferences about those matrices in sufficient generality, we must first formally define the class of sub-Gaussian matrices with random covariance matrices. 
\begin{definition}~\label{def:sub-Gaussian_rnd}
For $j\in [n^{(i)}]$ and $i\in [m] $, we say $X_j^{(i)}\in \real^p$ comes from a family of random $p$-dimensional multivariate sub-Gaussian with randomised covariance matrices $\lc\Gcal^{(i)}\rc$ distribution of size $m$ and parameter $\sigma$ if  
\begin{enumerate}
    \item $\expec[X_j^{(i)}|\Gcal^{(i)}]=0$, $\Cov(X_j^{(i)}|\Gcal^{(i)})=\Gcal^{(i)}$.
    \item Given $\Gcal^{(i)}$, $\frac{X_{l,j}^{(i)}}{\Gcal_{l,l}^{(i)}}$ is sub-Gaussian with parameter $\sigma$, $\forall i\in [m]$ and $j\in [n^{(i)}]$ and $l\in[p]$.
    \item Given $\Gcal^{(i)}$, $X_j^{(i)}$ are independent random variables.
\end{enumerate}
\end{definition}
\begin{remark}
We would like to point out that such definitions are standard in the literature and have been used extensively in prior work (see, for example, \cite{zhang2021meta}) for modelling sub-Gaussian random variables with randomised covariance matrices under the meta-learning setup. We have deferred the definition of Sub-Gaussian and Sub-exponential distributions to \Cref{app:desiderata}. 
\end{remark}

\paragraph{Problem Formulation. }\label{sec:formulation}
In this paper, we have focused on estimating the support of the PC matrix of a multivariate sub-Gaussian distribution. We first estimate a superset of the support of the PC matrices. Then we solve a novel task utilising the recovered support union. 

To be precise, let for $i\in [m]$, $\Sigma$ be the true underlying covariance matrix and  $\Sigma^{(i)}$ be perturbations of $\Sigma$, with each $\Sigma^{(i)}$ forming the covariance matrix of an auxiliary task $i$. Then, we are interested in using those auxiliary tasks to recover support union $J$ of the PC matrix $\Pi$ corresponding to $\Sigma$. It is noteworthy that by support union we refer to $supp(\Pi)$ and not a union of supports of each individual auxiliary task. This is because we are interested in a novel task where the support of its PC matrix is assumed to be a subset of the support of $\Pi$. This allows us to achieve a substantial reduction in sample complexity from $O(\log(p))$ to $O(\log |J|)$ while estimating the support in the novel task. 

\section{Our Meta-Learning Model}
In this section, we define a generative model for the population covariance matrices. Let $\Sigma$ be a positive semi-definite, real symmetric matrix. For some orthonormal matrix U, and positive diagonal matrix $\Lambda$, the spectral decomposition of $\Sigma$ gives us,
\begin{align*}
\Sigma = U\Lambda U^T := \sum_{l=1}^p \lambda_{l}(\Sigma)U_{*,l}U_{*,l}^T,    
\end{align*}
where $U_{*,l}$ are orthonormal vectors and $\lambda_i$ are the diagonal elements of $\Lambda$. If $\lambda_l(\Sigma)>\lambda_{l+1}(\Sigma)\forall\enspace l$, then, \begin{align*}~\label{eq:prin-comp}
    \Pi=\sum_{l=1}^k U_{*,l}U_{*,l}^T\numberthis
\end{align*} gives us the unique $k$-dimensional principal subspace of $\Sigma$. With random matrices $R^{(i)}$ and random matrices $D^{(i)}$, we generate a sequence of random matrices $\Sigma^{(1)},\dots,\Sigma^{(m)}\in \real^p$.
For each $i\in[m]$, the covariance matrix $\Sigma^{(i)}$ corresponding to the $i$-th auxiliary task is generated as,
\begin{align*}
\Sigma^{(i)}& =R^{(i)}U(\Lambda+D^{(i)})U^T\lp R^{(i)}\rp^T~\label{eq:model}\numberthis. 
\end{align*}
Referring to \Cref{def:sub-Gaussian_rnd}, we assume that for $j\in [n^{(i)}]$, our data, $X_{j}^{(i)}$ is a random p-dimensional multivariate sub-Gaussian random variable with randomised covariance matrix $\Sigma^{(i)}$. Let $S^{(i)}:=\frac{1}{n^{(i)}} \sum_{j=1}^{n^{(i)}} X_{j}^{(i)}(X_{j}^{(i)})^T$ be the sample covariance matrix corresponding to the auxiliary task $\Sigma^{(i)}$. Recall from \cite{mourtada2022improper} that an improper estimator of a parameter is defined to be any out-of-model predictor for that parameter. A ``pooled" estimator $S$ forms the improper estimation of $\Sigma$ as defined below:
\begin{align*}~\label{eq:sample_cov_pooled}
 S:=\frac{1}{m} \sum_{i=1}^m S^{(i)}.\numberthis
\end{align*}
\subsection{The Predictive Covariance Loss Function}
In this subsection, we describe the objective function that we maximise to obtain the improper estimator. Recall from Theorem 1 in \cite{overton1992sum} that PCA can be interpreted as a covariance maximisation technique maximising the predictive covariance $\tr\lp\Cov(X,HX|H)\rp$. Observe that,
\begin{align*}
      &\tr\lp\Cov\lp X, HX |H \rp\rp = \inner{\Sigma}{H},
\end{align*}
for $\Cov(X)=\Sigma$ and a rank $k$ orthogonal matrix $H$. Accordingly, the $k$-dimensional sparse principal component seeks the rank-$k$ orthogonal projection matrix $H$ such that for a penalty $\rho>0$ the penalised predictive covariance 
\begin{align*}~\label{eq:obj_func_main}
    \inner{\Sigma}{H}-\rho \|H\|_{1,1}\numberthis
\end{align*}
is maximised. Crucial to getting a consistent estimate of the principal component matrix is its identifiability. Unless the PC matrix is unique, there is no hope to correctly estimate the PC matrix. For example, if $\Sigma=\Ibb_p$, then every $k$-dimensional subspace of $\real^p$ is a valid principal component. Slightly abusing notation, by $\lambda_k$ we denote $\lambda_k(\Sigma)$ and by $\lambda_{k+1}$ we denote $\lambda_{k+1}(\Sigma)$. We make our first assumption about the uniqueness of the PC matrix, denoted in terms of the spectral gap of the covariance matrix $\Sigma$.
\begin{assumption}[Uniqueness]~\label{assume:unique} Let $\lambda_k$ and $\lambda_{k+1}$ be the $k$-th and $k+1$-th eigenvalues of $\Sigma$. Then,
$\lambda_k-\lambda_{k+1}>0$.
\end{assumption}
\begin{remark}
This condition ensures that the principal subspace is identifiable. Subsequently, we need to define the notion of sparsity of the PC matrix. Recall from \cref{eq:prin-comp} that $\Pi$ is a positive semi-definite matrix. Consequently, the $i$-th row/column of $\Pi$ is $0$ is non-zero if and only if $\Pi_{i,i}\neq 0$. Hence, it is sufficient to recover the set of non-zero elements in the diagonal of $\Pi$. Hence, we can make our sparsity assumption.
\end{remark}
\begin{assumption}[Sparsity]~\label{assume:sparse} Let $\Pi$ be the $k$-dimensional PC matrix corresponding to $\Sigma$. Then for some support $J$, we have $supp(\Pi)\subseteq J.$
\end{assumption}
Throughout \cref{sec:meta-support-recovery}, we will assume that $\Sigma$ satisfies both \Cref{assume:unique} and \Cref{assume:sparse}.
\subsection{The Improper Objective Function}
As illustrated in \Cref{sec:formulation}, we recover the support union by estimating the true covariance matrix. To be specific, we pool together all the samples from $m$ tasks and maximise the $l_1$-regularised predictive covariance between the true and the estimated covariance matrices, i.e., we solve the following optimisation problem with regularisation constant $\rho>0$
\begin{align*}~\label{eq:meta_objective}
    &\hat{\Pi}:=  \arg \sup_{H}\sum_{i=1}^{m} T_i \lc\inner{S^{(i)}}{H}-\rho \|H\|_{1,1}\rc \qquad \text{ subject to } H\in \Fcal^k, \text{ where, }\numberthis\\
    &\Fcal^k:=\{M\in \real^{p\times p}:0\preceq M \preceq \Ibb_p \text{ and } \tr(M)=k\},
\end{align*}
and $T_i$'s are weights proportional to the number of samples $n^{(i)}$ for each $i$. For clarity of exposition, we assume that $n^{(i)}=n$ for all values of $i$ for the reminder of this paper. This allows us to assume $T^{(i)}=\frac{1}{m}$ without losing generality. Consequently, we can rewrite the objective function in (\ref{eq:meta_objective}) as
\begin{align}~\label{eq:meta_objective_pooled}
    &\pihat=\arg\sup_{H}\inner{S}{H}-\rho \|H\|_{1,1}\text{ subject to } H\in \Fcal^k,
\end{align}
where $S$ is as defined in \cref{eq:sample_cov_pooled}. It is noteworthy that \Cref{eq:meta_objective} is an improper estimation technique since it estimates a single PC matrix with data from different distributions. This enables us to efficiently recover the true support with an optimal number of sample size per task (see \Cref{sec:meta-support-recovery}).

For the next step, suppose that we have successfully recovered the true support union $J$ in the first step. Then for the novel task, ($m+1$-th task), we can assume that the support of its PC matrix is also a subset of the recovered support union. Then, for a noisy estimate $S^{(m+1)}$ of $\Sigma^{(m+1)}$ we propose the following $l_1$-regularised predictive covariance loss function 
\begin{align*}
    &\inner{S^{(m+1)}}{H}-\rho \|H\|_{1,1} \text{ subject to } H\in \Fcal^k \text{ and } supp(H)\subseteq J.\numberthis~\label{eq:novel_obj}
\end{align*}
We note that \cref{eq:novel_obj} is also a form of improper estimation due to presence of extra constraint. This improper estimation technique reduces the dimension of the optimisation problem from $p\times p$  to $|J|\times |J|$.
\section{Support Union Recovery for the Auxiliary Tasks}\label{sec:meta-support-recovery}
In this section, we formally show that a sample complexity of $O\lp\frac{\log p}{m}\rp$ is sufficient for the recovery of the support union via optimising \Cref{eq:meta_objective_pooled}. One key task in achieving the said sample complexity is to carefully upper bound the tail probabilities of $\ninf{S-\Sigma}$. To that end, we prove the following tail probability for a family of multivariate sub-Gaussian with randomised covariance matrices. Let $S^{(i)}=\frac{1}{n}\sum_{t=1}^nX_t^{(i)}(X_t^{(i)})^{\text{T}}$.
\begin{lemma}~\label{lem:TailSampleCov}
	Let $\Sigma^{(i)}$ be a sequence of random covariance matrices such that $\ninf{\Sigma^{(i)}}\leq \eta$. Then, for $\{X_{j}^{(i)}\}_{1\le j\le n,1\le i\le m}$ following a family of random $p$-dimensional multivariate sub-Gaussian distributions with parameter $\sigma$ and random covariance matrices $\Sigma^{(i)}$, 
	\begin{small}
        \begin{align*}
        \prob\lp \ninf{\frac{1}{m}\sum \lp S^{(i)}-\Sigma^{(i)}\rp} \geq \epsilon \rp\leq \frac{p(p+1)}{2}e^{-\frac{nm\epsilon^2}{512\sigma^4\eta}},
        \end{align*}	
    \end{small}
    whenever $0<\epsilon<32\eta\sigma^2$.
\end{lemma}
The proof of this lemma is presented in \Cref{sec:TailSampleCov}. Recall from \cref{eq:model} that we defined each auxiliary task $\Sigma^{(i)}$ in terms of random matrices $R^{(i)}$ and random matrices $D^{(i)}$. Key to the process of improper estimation is the control of the maximal loss between the auxiliary tasks $\Sigma^{(i)}$ and the true underlying task $\Sigma$. In other words, we need to find theoretical guarantees for $\ninf{\Sigma^{(i)}-\Sigma}$. We achieve it by controlling the variance and support of the perturbations $R^{(i)}$ and $D^{(i)}$. To that end, we make the following assumption.
\begin{assumption}~\label{assume:unbiased} Assume that for each $1\leq i\leq m$, $R^{(i)}$ is a random matrix, and $D^{(i)}$ is a random matrix. Then,
\begin{enumerate}
    \item Each $R^{(i)}$ is independent and identically distributed. Moreover, for a constant $\constant_R\geq 0$,
    \begin{align}
        &\expec\|R^{(i)}-\Ibb_p\|_{\infty,1}^2\leq \frac{\constant_R}{p}\text{ and,}\label{assume:rotation1}\\
        &\|R^{(i)}-\Ibb_p\|_{\infty,1}\leq \constant_R \text{ almost surely.}\label{assume:rotation2}
    \end{align}
    \item Each $D^{(i)}$ is independent and identically distributed with $\expec[D^{(i)}]=0$. Moreover, for some constant $L>0$, 
    \begin{align}
        \lambda_1(D^{(i)})<L\text{     almost surely.}~\label{assume:translation1}
    \end{align} 
\end{enumerate}
\end{assumption}
\begin{remark}
White noise additive models have been used extensively in prior literature to study precision and covariance matrices \cite{wang2021sample,zhang2021meta}. It is noteworthy that under our assumptions, we can recover additive models from \cref{eq:model} when $R^{(i)}=\Ibb_p$ for every $i$. However, $R^{(i)}$ suffers from curse of dimensionality. \Cref{assume:rotation1} provides a necessary shrinkage property to $R^{(i)}$. It is noteworthy that such a property is necessary to ensure consistency of $\frac{1}{m}\sum_{i=1}^m \Sigma^{(i)}$ as an estimator of $\Sigma$. Our final technical assumption is given below.
\end{remark}
\begin{assumption}~\label{assume:corr-con}
If $\Sigma$ is a covariance satisfying \cref{assume:sparse}. Then, 
\begin{align*}
    \frac{8|J|}{\lambda_k-\lambda_{k+1}}\lVert\Sigma_{J^c, J}\rVert _{2,\infty}<1.
\end{align*}
\end{assumption}

This condition previously appears in \cite{lei2015sparsistency} and recovers the conditions by \cite{amini2008high} when $\Sigma_{J^c, J}=0$. It is also similar to the irrepresentability condition in the seminal papers \cite{zhao2006model,meinshausen2006high,meinshausen2009lasso} for $l_1$ penalised sparse regression and lasso type recovery.  As far as we know, this is the most general condition for the support recovery of PC matrices and encompasses the conditions when $\Sigma$ is block diagonal. For ease of exposition, we define the following notation:
\begin{align*}
    & \lambda_{diff}:=\lambda_k-\lambda_{k+1},
    & \lambda^{\dagger}:=2\lp\constant_r+1\rp^2\lp\lambda_1(\Sigma)+L\rp,\\
    & \rho_1 := 4\lambda^{\dagger}\sqrt{\frac{\log(p)}{m}}+\frac{\constant_R}{p}\lambda_1(\Sigma)+2\sqrt{\frac{\constant_R}{p}}\lambda_1(\Sigma),
    &\rho_2 := 16\sqrt{2\sigma^4\lambda^{\dagger} \frac{\log\lp p+1 \rp}{nm}}.
\end{align*}
The following theorem establishes the $O(\frac{\log p}{m})$ sample complexity for the support union recovery from multiple auxiliary tasks.
\begin{theorem}[Support Union Recovery]~\label{thm:meta}
Let \cref{assume:unbiased} and \cref{assume:corr-con} hold. Then there exists a large enough constant $\constant_{\pi}>0$ and a constant $\alpha\in (0,1)$ such that, whenever 
\begin{small}
\begin{align*}
&\alpha\max \lc\rho_1, \rho_2\rc < \rho < \min\lc\frac{\lp\lambda_{diff}\rp\min_{j\in J}{\Pi_{j,j}}}{16|J|},\frac{\lp\lambda_{diff}\rp^2}{4|J|\lp\lambda_{diff}+8\lambda_1(\Sigma)\rp}\rc  
\end{align*}
\end{small}
then, with probability at least {$1-\frac{2}{p^2}-\frac{1}{2(p+1)^2}$} the solution $\hat \Pi$ to the objective function \Cref{eq:meta_objective_pooled} satisfies,
\begin{align*}
    &1.\enspace supp(\hat{\Pi})=J, 
    & 2. \ninf{\hat\Pi-\Pi}\leq \constant_{\pi}\rho.
\end{align*}

\end{theorem}
In other words, \cref{thm:meta} gives us our required sample complexity for sign consistency and error bound.
\subsection{Sketch of Proof for Theorem \ref{thm:meta}}\label{sec:proof-sketch-meta}
The complete proof of this theorem may be found at \cref{sec:proof-meta} and we only make the outline here. The proof of \Cref{thm:meta} is via primal-dual witness method. It starts by writing down the objective function \cref{eq:meta_objective_pooled}.
\begin{align*}
   & \max_{H\in\Fcal^k}\inner{S}{H}-\rho \|H\|_{1,1}.
\end{align*}
Define the subspace
\begin{small}
\[\bbb_p:= \lc Z \in \real^{p\times p}:diag(Z)=0, Z=Z^T, \ninf{Z}\leq 1\rc.
\]
\end{small}

Using strong-duality of $\|\cdot\|_{1,1}$, we can rewrite the objective function as,
\begin{align*}
    \max_{\substack{H\in \Fcal^k\\ Z\in \bbb}}\inner{S}{H}-\rho \inner{H}{Z+\Ibb_p}
    \equiv \ \max_{\substack{H\in \Fcal^k \\Z\in \bbb}}\inner{S}{H}-\rho \inner{H}{Z} -\rho k
    \equiv \ & \max_{\substack{H\in \Fcal^k\\ Z\in \bbb}}\inner{H}{S-\rho Z},\numberthis\label{eq:proof-sketch-dual}
\end{align*}
where the last step follows from the fact that $k$ is a constant. To maximise \Cref{eq:proof-sketch-dual}, we need to find solutions $(\hat H,\hat Z)$ such that Karush-Kuhn-Tucker (KKT) \cite{boyd2004convex} conditions 
\begin{align}
    \hat{Z}_{i,j} = sign(\hat H_{i,j}) & \forall \enspace i\neq j, \hat H_{i,j}\neq 0,~\label{eq:kkt-1}\\
    \hat{Z}_{i,j} \in [-1,1]  \forall & \enspace i\neq j, \hat H_{i,j} = 0,~\label{eq:kkt-2}\\
    \hat H =\arg \max & \inner{S-\rho \hat Z}{H},~\label{eq:kkt-3}
\end{align}
are satisfied. Proceeding according to the primal-dual witness method, we construct the constrained optimisation problem where the $supp(H)\subseteq J$.
\begin{align*}
    \max_{H\in\Fcal^k,supp(H)\subseteq J}\inner{S}{H}-\rho \|H\|_{1,1}.\numberthis~\label{eq:constrained-objective}
\end{align*}
Let $(\tilde H,\tilde Z)$ be the corresponding primal and dual variables. The proof selects an appropriate $|J|\times |J|$ orthonormal matrix $Q$ according to 
\cref{lem:sparse-solution} with the correct rotation such that the Frobenius norm between the true eigenvectors $U_{J,*}$ and the estimated eigenvectors $\hat U_{J,*}$ is small. Then we propose the following solution $(\hat H,\hat Z)$ to the dual objective \cref{eq:proof-sketch-dual}. It is noteworthy that this construction is similar to that in Theorem 1 in \cite{lei2015sparsistency}.
\begin{align}
 &   \hat{H} = \tilde H,
\nonumber
\label{eq:dual-construct-1}\\
&\hat{Z}_{J,J}= \tilde Z_{J,J},
\\
\label{eq:dual-construct-2}
&\hat{Z}_{ij}= \frac{1}{\rho} \bigl\{S_{ij}-\langle
Q_{i,*}, \Sigma_{J,j} \rangle \bigr\},  (i,j)\in J\times
J^c,
\\
\label{eq:dual-construct-3}
&\hat{Z}_{ij} =\frac{1}{\rho} (S-\Sigma)_{ij},  (i,j)\in
J^c\times J^c, i\neq j.
\end{align}

The rest of the proof follows in $4$ parts.
\begin{enumerate}
 \item Finding an appropriate $Q$ matrix and an appropriate tail probability bound under Assumption \ref{assume:unbiased} by using a combination of Matrix Chernoff bound (see \cref{app:desiderata}) and \Cref{lem:TailSampleCov}.
 \item Proving that the proposed solution satisfies the KKT conditions (\ref{eq:kkt-1}), (\ref{eq:kkt-2}), and (\ref{eq:kkt-3}) with high probability for our choices of $\rho$ and is unique.
  \item Find an upper bound for estimation error. We achieve this through the following upper bound, $\ninf{\hat H - \Pi}<2\frac{\rho|J|}{\lambda_{diff}}$, which is a mild variation of a similar upper bound from Lemma 2 in \cite{vu2013fantope}. Thus, for sufficiently small enough $\rho$, we prove that $sign(\hat \Pi)=sign(\Pi)$, and our estimate is sign consistent.
  \item Showing that the proposed solution is sign consistent using a similar technique to part 3.
\end{enumerate}
\section{Support Recovery for the Novel Task}~\label{sec:novel-support-recovery}

In this section, we show that $O\lp\log |J|\rp$ sample complexity is sufficient for the support recovery for a novel task where $J$ is the recovered support union in Section \ref{sec:meta-support-recovery}. We also show that for the novel task, the estimation problem can be reduced from estimating a $p\times p$ dimensional matrix to a $|J|\times |J|$ dimensional matrix. Let $n^{(m+1)}$ be the number of samples in the novel task. We begin by establishing the reduction in dimension. Let $\Sigma^{(m+1)}$ be the variance matrix corresponding to the novel task with spectral decomposition
\begin{align*}
    U^{(m+1)}\Lambda^{(m+1)}\lp U^{(m+1)}\rp^T,
\end{align*}
and the $k$-dimensional principal component matrix $\Pi^{(m+1)}$. Throughout the rest of this section, we assume that $\Sigma^{(m+1)}$ satisfies \Cref{assume:unique} and \Cref{assume:sparse}. We write $\Sigma^{(m+1)}$ and its estimator $S^{(m+1)}$ as the following block matrices
\begin{align}~\label{eq:block-cov-est}
    \Sigma^{(m+1)} = \begin{bmatrix}
    \Sigma_{J,J}^{(m+1)} & \Sigma_{J,J^c}^{(m+1)}\\
    \Sigma_{J^c,J}^{(m+1)} & \Sigma_{J^c,J^c}^{(m+1)}
    \end{bmatrix}, & \quad S^{(m+1)} =
    \begin{bmatrix}
    S_{J,J}^{(m+1)} & S_{J,J^c}^{(m+1)}\\
    S_{J^c,J}^{(m+1)} & S_{J^c,J^c}^{(m+1)}
    \end{bmatrix}.
\end{align}
We obtain the $k$-dimensional principal component of the novel task $\Pi^{(m+1)}$ by optimising the objective function,
\begin{align*}
    \max_{H\in \Fcal^k}\inner{S^{(m+1)}}{H}-\rho \|H\|_{1,1}.~\numberthis~\label{eq:novel_objective}
\end{align*}
However, we have already recovered the support union $J$. Hence, we can add another constraint on $H$: $supp(H)=J$. Moreover, since the row and column $i$ of the principal component $\Pi^{(m+1)}$ is $0$ if and on $\Pi_{i,i}^{(m+1)}=0$, we can rewrite this additional constraint as,
\begin{align*}
    H= \begin{bmatrix}
    H_{J,J} & 0\\
    0 & 0
    \end{bmatrix}.
\end{align*}
Since $H\in \Fcal^k$, the maximisation problem only makes sense if $|J|>k$. This also automatically implies that $H\in \Fcal_J^k$, where $\Fcal_J^k := \lc M\in \real^{|J|\times |J|} : 0\preceq M \preceq \Ibb_{|J|} \text{ and } \tr(M)=k \rc$.
Therefore, the objective corresponding to the novel task maximises $\inner{S^{(m+1)}}{H}-\rho \|H\|_{1,1}$ subject to $H\in \Fcal^k$ and $supp(H)\subseteq J$.
This is equivalent to maximising
\begin{align*}
    {\hat\Pi^{(m+1)}}:=\arg\sup_{H_{J,J}}\ \inner{S_{J,J}^{(m+1)}}{H_{J,J}}-\rho \|H_{J,J}\|_{1,1}, \text{ for  } H_{J,J}\in \Fcal_J^k~\numberthis
\end{align*}
where $supp(\hat{\Pi})=J$ from \cref{thm:meta} is the recovered support union, and $H_{J,J},S_{J,J}\in\real^{|J|\times|J|}$. The following proposition on the eigenvalues and eigenvectors of $\Sigma_{J,J}^{(m+1)}$ shows that the estimation problem can be reduced from estimating a $p\times p$ dimensional matrix to a $|J|\times |J|$ dimensional matrix. 
\begin{proposition}~\label{prop:sub_eigen}
For $1\leq l\leq k$, $U_{l,J}^{(m+1)}$ are the first $k$ eigenvectors of $\Sigma_{J,J}^{(m+1)}$ with corresponding eigenvalues $\lambda_l^{(m+1)}$.
\end{proposition}
The proof of this proposition is given in \Cref{sec:proof-sub-eigen}. Following from the fact that $\Sigma^{(m)}$ satisfies \Cref{assume:unique} and \Cref{assume:sparse}, \cref{prop:sub_eigen} establishes as an immediate consequence that 
\begin{align*}
    & \lambda_k^{(m+1)}-\lambda_{k+1}^{(m+1)}>0, \text{ and, } supp(\Pi^{(m+1)})\subseteq J^{(m+1)}.
\end{align*}
This establishes that the $k$-dimensional PC matrix of $\Sigma_{J,J}^{(m+1)}$ is sparse and unique, as well as the solution to a lower dimensional objective function.

Overloading notation, define $\lambda_{diff}^{(m+1)}:=\lambda_k^{(m+1)}-\lambda_{k+1}^{(m+1)}$. We make the following technical assumption on the elements of $\Sigma_{J,J}^{(m+1)}$. 
\begin{assumption}~\label{assume:novel-corr-cond} Let $\Sigma^{(m+1)}$ be the covariance matrix corresponding to the novel task and $J$ be the recovered support union. Then,
\begin{align*}
    \frac{8|J^{(m+1)}|}{\lambda_{diff}^{(m+1)}}\lVert \Sigma^{(m+1)}_{(J^{(m+1)})^c, J^{(m+1)}}\rVert _{2,\infty}<1.
\end{align*}
\end{assumption}
We define 
\begin{align*}
    &\constant_{n1}:=\frac{1}{32\lambda_1\lp\Sigma^{(m+1)}\sigma^2\rp},\quad  \constant_{n2}:=\frac{4|J^{(m+1)}|\lp 1+
    \frac{8\lambda
    _1^{(m+1)}}{\lambda_k^{(m+1)}-\lambda_{k+1}^{(m+1)}}\rp}{\lambda_k^{(m+1)}-\lambda_{k+1}^{(m+1)}},\\
    &\constant_{n3}:=\frac{\lp\lambda_k^{(m+1)}-\lambda_{k+1}^{(m+1)}\rp\min_{j\in\lb|J^{(m+1)}|\rb}\Pi_{j,j}^{(m+1)}}{4|J^{(m+1)}|}.
\end{align*}The following theorem shows the $O(\log |J|)$ sample complexity for support recovery of the novel task.
\begin{theorem}[Novel Task Support Recovery]~\label{thm:novel} Let $\Sigma^{(m+1)}$ be the covariance matrix corresponding to the novel task satisfying \Cref{assume:novel-corr-cond} and $J$ be the recovered support union. Then, there exists a constant $\alpha^{(m+1)}\in(0,1)$ such that if
\begin{align*}
 &n^{\frac{1}{3}}>\max\lc\constant_{n1},\constant_{n2},\constant_{n3}\rc \text{ and, } \frac{1}{ n^{1/3}\alpha^{(m+1)}}<\rho< \frac{\lambda_k^{(m+1)}-\lambda_{k+1}^{(m+1)}}{4|J^{(m+1)}|\lp 1+\frac{8\lambda
_1^{(m+1)}}{\lambda_k^{(m+1)}-\lambda_{k+1}^{(m+1)}}\rp},
\end{align*}
then, with probability, $1-\frac{|J|(|J|+1)}{2}e^{-\frac{n^{1/3}}{8\constant\sigma^2\lp2\lambda_1(\Sigma^{(m+1)})\rp}}$, the maximiser $\hat{\Pi}^{(m+1)}$ to the objective function \Cref{eq:novel_obj}  satisfies, $supp(\hat{\Pi}^{(m+1)})=J^{(m+1)}$.
\end{theorem}
The proof of this theorem is deferred to \Cref{sec:proof-novel}
\section{Experiments}
In this section, we empirically validate our theoretical claims in Sections \ref{sec:meta-support-recovery} and \ref{sec:novel-support-recovery}. 
\paragraph{Gaussian Distribution Setting.}\label{sec:gauss-err}
For all the experiments in this subsection, we let $p=50$, $k=5$ and $|J|=5$ and perform 100 repetitions of each setting. We first consider the setting of different sample sizes. We choose $n\in \lc 3,5,7,9\rc$ and use $\rho = \sqrt{\frac{\log (p+1)}{mn}}$ for all pairs of $(m,n)$. The number of tasks $m$ is rescaled to $T$ defined by $\frac{mn}{\log (p+1)}$. For each $j\in[n]$, $X^{(i)}_j$ is generated
from a Gaussian distribution with mean a $p$-dimensional $0$ vector and covariance matrix $\Sigma^{(i)}$. \Cref{fig:small-dim} depicts the outcome of our experiments. For different choices of $n$, the graphs overlap each other perfectly (both $\prob(\hat{J}=J)$ and $\|{\hat{\Pi}-\Pi}\|_{\infty,\infty}$). Then, we took the recovered support union and derived the probability of support recovery for a novel task. 

{We perform the experiments when there are $2,3$, or $4$ extra zeros and all of those have greater than $95\%$ probability of accurately identifying the extra zeros. Moreover, in \Cref{sec:additional-exp} we provide more details and show extra results for the uniform and the mixture of sub-Gaussian settings.
\paragraph{Real World Experiments.} In \Cref{app:real-world}, we show that our method obtains anywhere between $85-95\%$ accuracy in a real world brain-imaging dataset.}
\begin{figure}[!h]
    \centering
    \subfloat[Probability of Support Union Recovery]{{\includegraphics[width=0.3\textwidth]{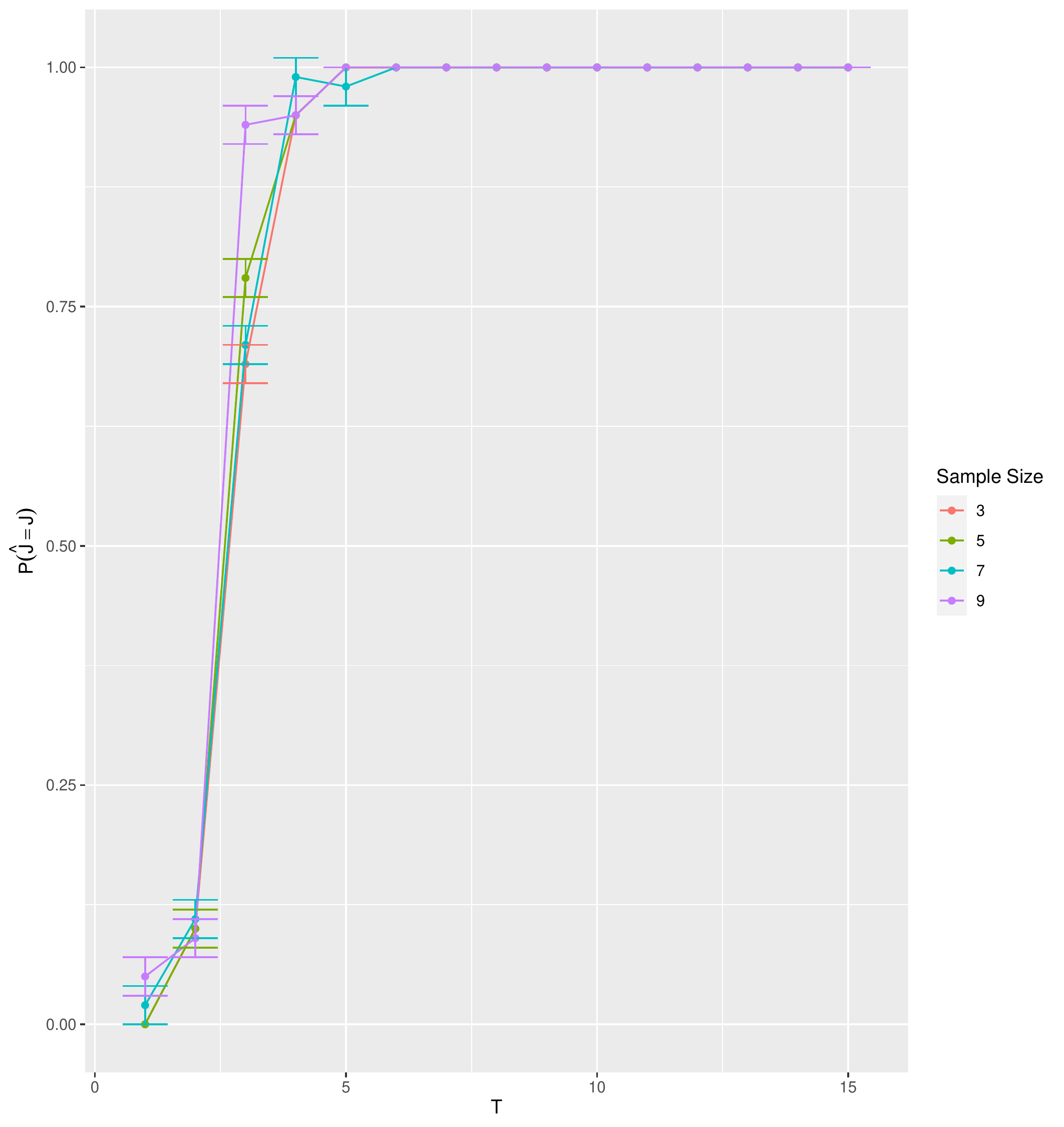}}}
    \subfloat[Maximal Error Bound]{{\includegraphics[width=0.3\textwidth]{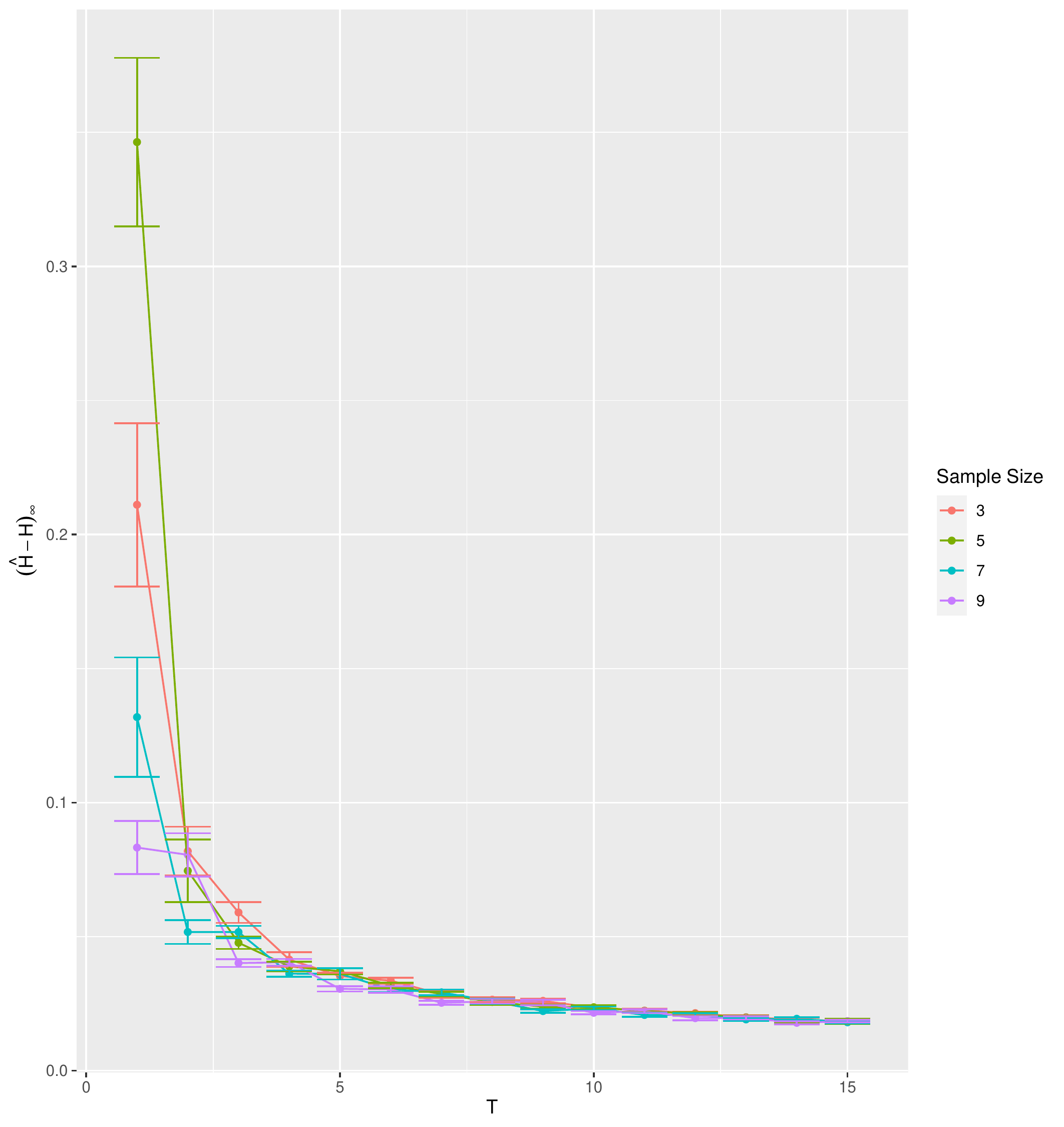}}}
    \subfloat[Support Recovery of Novel Task]{\includegraphics[width=0.35\textwidth]{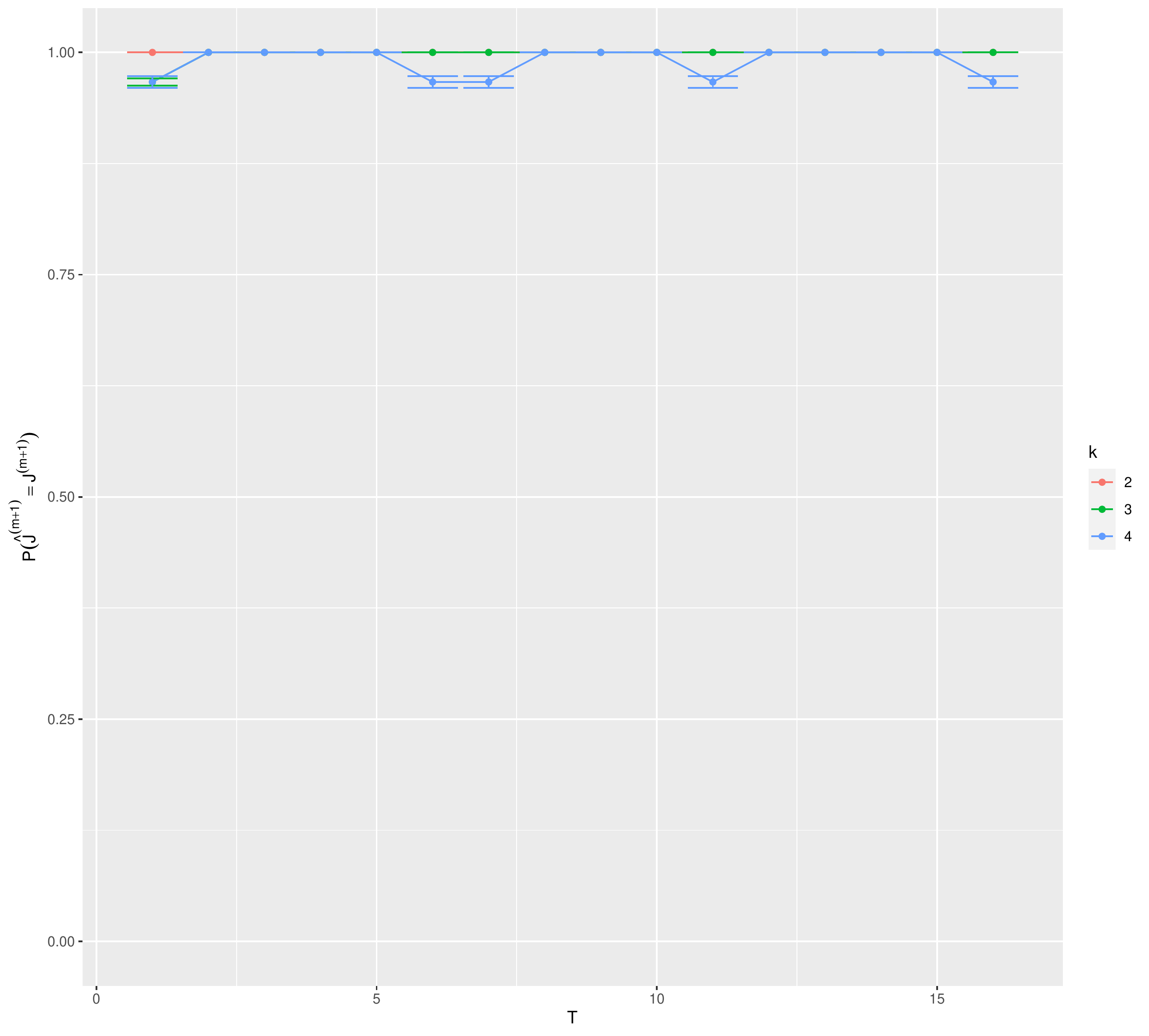}}
    \caption{Left, Center: Simulations for \Cref{thm:meta} on the probability of support union recovery and maximal error under various settings of $n$. $\rho = \sqrt{\frac{\log (p+1)}{mn}}$. The x-axis is set by $T:=\frac{mn}{\log(p+1)}$ varying from $\lc1,\dots,15\rc$ for support union recovery. Right: Probability of support recovery for novel task. $\rho = \sqrt{\frac{\log (|J|+1)}{n}}$. The x-axis is set by $T:=\frac{n}{\log(|J|+1)}$ varying from $\lc10,\dots,25\rc$ }
    \label{fig:small-dim}
\end{figure}
\section{Future Studies}
There still exists some open problems regarding Meta Sparse PCA that were not addressed in this work. Dependencies in the data, such as when the data is from a time series makes an interesting such case we do not address. Similarly, it is also unclear how estimation should proceed when there is missing data. Finally, we would like to point out that there are newer techniques such as Kernel PCA which can describe non-linear boundaries and it is unclear how Meta sparse PCA matrices should proceed under such circumstances. 

\bibliography{biblio}

\begin{thebibliography}{10}

\bibitem{amini2008high}
Arash~A Amini and Martin~J Wainwright.
\newblock High-dimensional analysis of semidefinite relaxations for sparse
  principal components.
\newblock In {\em 2008 IEEE international symposium on information theory},
  pages 2454--2458. IEEE, 2008.

\bibitem{boyd2004convex}
Stephen Boyd, Stephen~P Boyd, and Lieven Vandenberghe.
\newblock {\em Convex optimization}.
\newblock Cambridge university press, 2004.

\bibitem{deshpande2014it}
Yash Deshpande and Andrea Montanari.
\newblock Information-theoretically optimal sparse pca.
\newblock In {\em 2014 IEEE International Symposium on Information Theory},
  pages 2197--2201, 2014.

\bibitem{honorio2014tight}
Jean Honorio and Tommi Jaakkola.
\newblock Tight bounds for the expected risk of linear classifiers and
  pac-bayes finite-sample guarantees.
\newblock In {\em Artificial Intelligence and Statistics}, pages 384--392.
  PMLR, 2014.

\bibitem{jenatton2010structured}
Rodolphe Jenatton, Guillaume Obozinski, and Francis Bach.
\newblock Structured sparse principal component analysis.
\newblock In {\em Proceedings of the Thirteenth International Conference on
  Artificial Intelligence and Statistics}, pages 366--373. JMLR Workshop and
  Conference Proceedings, 2010.

\bibitem{johnstone2009consistency}
Iain~M Johnstone and Arthur~Yu Lu.
\newblock On consistency and sparsity for principal components analysis in high
  dimensions.
\newblock {\em Journal of the American Statistical Association},
  104(486):682--693, 2009.

\bibitem{lei2015sparsistency}
Jing Lei and Vincent~Q Vu.
\newblock Sparsistency and agnostic inference in sparse pca.
\newblock {\em The Annals of Statistics}, 43(1):299--322, 2015.

\bibitem{lin2008development}
Chin-Teng Lin, Yu-Chieh Chen, Teng-Yi Huang, Tien-Ting Chiu, Li-Wei Ko,
  Sheng-Fu Liang, Hung-Yi Hsieh, Shang-Hwa Hsu, and Jeng-Ren Duann.
\newblock Development of wireless brain computer interface with embedded
  multitask scheduling and its application on real-time driver's drowsiness
  detection and warning.
\newblock {\em IEEE Transactions on Biomedical Engineering}, 55(5):1582--1591,
  2008.

\bibitem{lounici2013sparse}
Karim Lounici.
\newblock Sparse principal component analysis with missing observations.
\newblock In {\em High dimensional probability VI}, pages 327--356. Springer,
  2013.

\bibitem{ma2013sparse}
Zongming Ma.
\newblock Sparse principal component analysis and iterative thresholding.
\newblock {\em The Annals of Statistics}, 41(2):772--801, 2013.

\bibitem{meinshausen2006high}
Nicolai Meinshausen and Peter B{\"u}hlmann.
\newblock High-dimensional graphs and variable selection with the lasso.
\newblock {\em The annals of statistics}, 34(3):1436--1462, 2006.

\bibitem{meinshausen2009lasso}
Nicolai Meinshausen and Bin Yu.
\newblock Lasso-type recovery of sparse representations for high-dimensional
  data.
\newblock {\em The annals of statistics}, 37(1):246--270, 2009.

\bibitem{mourtada2022improper}
Jaouad Mourtada and St{\'e}phane Ga{\"\i}ffas.
\newblock An improper estimator with optimal excess risk in misspecified
  density estimation and logistic regression.
\newblock {\em J. Mach. Learn. Res.}, 23:31--1, 2022.

\bibitem{nadler2008finite}
Boaz Nadler.
\newblock Finite sample approximation results for principal component analysis:
  A matrix perturbation approach.
\newblock {\em The Annals of Statistics}, 36(6):2791--2817, 2008.

\bibitem{overton1992sum}
Michael~L Overton and Robert~S Womersley.
\newblock On the sum of the largest eigenvalues of a symmetric matrix.
\newblock {\em SIAM Journal on Matrix Analysis and Applications}, 13(1):41--45,
  1992.

\bibitem{park2019sparse}
Seyoung Park and Hongyu Zhao.
\newblock Sparse principal component analysis with missing observations.
\newblock {\em The Annals of Applied Statistics}, 13(2):1016--1042, 2019.

\bibitem{paul2007asymptotics}
Debashis Paul.
\newblock Asymptotics of sample eigenstructure for a large dimensional spiked
  covariance model.
\newblock {\em Statistica Sinica}, pages 1617--1642, 2007.

\bibitem{persichetti2021data}
Andrew~S Persichetti, Joseph~M Denning, Stephen~J Gotts, and Alex Martin.
\newblock A data-driven functional mapping of the anterior temporal lobes.
\newblock {\em Journal of Neuroscience}, 41(28):6038--6049, 2021.

\bibitem{pollardprocesses}
David Pollard.
\newblock {\em PROCESSES: THEORY AND APPLICATIONS}.
\newblock 1984.

\bibitem{qi2013sparse}
Xin Qi, Ruiyan Luo, and Hongyu Zhao.
\newblock Sparse principal component analysis by choice of norm.
\newblock {\em Journal of multivariate analysis}, 114:127--160, 2013.

\bibitem{sun2019image}
Yaoqi Sun, Liang Li, Liang Zheng, Ji~Hu, Wenchao Li, Yatong Jiang, and
  Chenggang Yan.
\newblock Image classification base on pca of multi-view deep representation.
\newblock {\em Journal of Visual Communication and Image Representation},
  62:253--258, 2019.

\bibitem{tropp2015introduction}
Joel~A Tropp.
\newblock An introduction to matrix concentration inequalities.
\newblock {\em arXiv preprint arXiv:1501.01571}, 2015.

\bibitem{vershynin2010introduction}
Roman Vershynin.
\newblock Introduction to the non-asymptotic analysis of random matrices.
\newblock {\em arXiv preprint arXiv:1011.3027}, 2010.

\bibitem{vu2013fantope}
Vincent~Q Vu, Juhee Cho, Jing Lei, and Karl Rohe.
\newblock Fantope projection and selection: A near-optimal convex relaxation of
  sparse pca.
\newblock {\em Advances in neural information processing systems}, 26, 2013.

\bibitem{wang2021sample}
Zhanyu Wang and Jean Honorio.
\newblock The sample complexity of meta sparse regression.
\newblock In {\em International Conference on Artificial Intelligence and
  Statistics}, pages 2323--2331. PMLR, 2021.

\bibitem{zhang2021meta}
Qian Zhang, Yilin Zheng, and Jean Honorio.
\newblock Meta learning for support recovery in high-dimensional precision
  matrix estimation.
\newblock In {\em International Conference on Machine Learning}, pages
  12642--12652. PMLR, 2021.

\bibitem{zhang2008global}
Taiping Zhang, William~S Chandler, James~M Hoell, David Westberg, Charles~H
  Whitlock, and Paul~W Stackhouse.
\newblock A global perspective on renewable energy resources: Nasa’s
  prediction of worldwide energy resources (power) project.
\newblock In {\em Proceedings of ISES World Congress 2007 (Vol. I--Vol. V)},
  pages 2636--2640. Springer, 2008.

\bibitem{zhao2006model}
Peng Zhao and Bin Yu.
\newblock On model selection consistency of lasso.
\newblock {\em The Journal of Machine Learning Research}, 7:2541--2563, 2006.

\bibitem{zheng2015attribute}
Hao Zheng, Zhanlei Yang, Liwei Qiao, Jianping Li, and Wenju Liu.
\newblock Attribute knowledge integration for speech recognition based on
  multi-task learning neural networks.
\newblock In {\em Sixteenth Annual Conference of the International Speech
  Communication Association}, 2015.

\bibitem{zou2006sparse}
Hui Zou, Trevor Hastie, and Robert Tibshirani.
\newblock Sparse principal component analysis.
\newblock {\em Journal of computational and graphical statistics},
  15(2):265--286, 2006.

\end{thebibliography}
\bibliographystyle{plain}
\newpage
\appendix

\section{Technical Desiderata}~\label{app:desiderata}
{In this section of the appendix, we state some technical definitions and lemmas.
We begin with the definition of a real valued sub-Gaussian random variable \cite{zhang2021meta}.
\begin{definition}
A random variable $X$ is called sub-Gaussian with parameter $\sigma\geq 0$ if, $\forall \rho\in \mathbb{R}$ 
\begin{equation}~\label{eq:sub-Gaussian}
    \expec[e^{\rho X}]\leq e^{\frac{\rho^2\sigma^2}{2}}.
\end{equation}
\end{definition}
As noted in Example 5.8 \cite{vershynin2010introduction}, all bounded or Gaussian random variables are also sub-Gaussian. We can extend the definition of univariate sub-Gaussian distributions to define a general class of multivariate  sub-Gaussian distribution on $\real^p$. Let $\Sigma\in \real^{p\times p}$ be a positive semi-definite, real symmetric matrix. 
\begin{definition}
We say that a random vector $X\in \real^p$ follows a multivariate sub-Gaussian distribution with covariance matrix $\Sigma$ and parameter $\sigma$ if 
\begin{enumerate}
    \item $\expec[X]=0$, and $\Cov(X)=\Sigma$.
    \item For $1\leq i\leq p$, $\frac{X_i}{\sqrt{\Sigma_{i,i}}}$ is a sub-Gaussian random variable with parameter $\sigma$.
\end{enumerate}
\end{definition}
Our last definition of this section is the sub-exponential distribution \cite{pollardprocesses}. 
\begin{definition}
A random variable $X$ with mean $0$ is said to follow a sub-exponential distribution with parameters $(\eta,b)$ if, 
\begin{align*}
    \expec[e^{\lambda X}]\leq e^{\frac{\lambda^2\eta^2}{2}} \text{ for } |\lambda|<\frac{1}{b}.
\end{align*}
\end{definition}
\begin{remark}
It is well established that (see for instance Equation (6) in \cite{honorio2014tight}) the square of a sub-Gaussian random variable with parameter $\sigma$ is sub-exponential with parameters $(4\sqrt{2}\sigma^2,4\sigma^2)$. 
\end{remark}
Next, we recall the matrix Chernoff bound (Theorem 5.1.1 in \cite{tropp2015introduction}).
}
\begin{lemma}~\label{lemma:matrix-hoeff}
Let $A_i$ be random matrices such that $\expec[A_i]=0$. If $\ninf{A_i}\leq \gamma$ almost surely, then for every $\epsilon>0$,
\begin{align*}
    \prob\lp \ninf{\frac{1}{m}\sum A_i}\geq \epsilon \rp\leq 2pe^{-\frac{3\epsilon^2m}{8\gamma^2}}.
\end{align*}
\end{lemma}
We then recall the following lemma from \cite{lei2015sparsistency}.
\begin{lemma}
\label{lem:sparse-solution}
Under the assumptions in Theorem~\ref{thm:meta}, let $\tilde H$ be
the solution
to the further constrained problem \eqref{eq:constrained-objective}.
Then $\tilde H$ is rank $k$ and unique.
Furthermore, there
exist $|J|\times k$ orthonormal matrices $U_J$, $\hat U_J$ such that:
\begin{enumerate}
\item $\begin{bmatrix}
    U_J\\
    0
\end{bmatrix}$ and $\begin{bmatrix}
  \hat  U_J\\
    0
\end{bmatrix}$ span the $k$-dimensional principal subspaces of $\Sigma$ and
$S-\rho\tilde Z$,
respectively.
\item
There exists a $|J|\times |J|$ orthonormal matrix $Q$ such that
\begin{eqnarray*}
\hat U_J & = & Q U_J,
\\
\lVert Q-\Ibb_{|J|} \rVert_F & \le&\frac{8\rho |J|}{\lambda_{diff}}.
\end{eqnarray*}
\end{enumerate}
\end{lemma}
We also recall the following lemma about the principal subspace projection of a matrix. Theorem 1 in \cite{overton1992sum}.
\begin{lemma}~\label{lem:fantop-projection}
Let $A$ be a symmetric matrix with eigenvalues $\gamma_1 \geq\cdots
\geq\gamma_p$
and orthonormal eigenvectors $v_1,\ldots,v_p$.

$\max_{H\in\mathcal{F}^{k}} \langle A, H \rangle
= \gamma_1 + \cdots+ \gamma_{k}$ and the maximum is achieved by
the projector of a $k$-dimensional principal subspace of
$A$. Moreover, the maximiser is unique if and only if $\gamma_k>\gamma_{k+1}$.
\end{lemma}

We also produce the following lemma along with the proof on the upper bound of the infinite norm of the product of two matrices.
\begin{lemma}~\label{lem:norm-bound}
Let $A$ and $B$ be two matrices with conformable dimensions. Then,
\begin{align*}
    \ninf{AB}\leq  \ninf{A}\|B\|_{1,\infty}
\end{align*}
\end{lemma}
\begin{proof}
\begin{align*}
    \ninf{AB}  = \max_{i,j\in [p]} |\sum_{k\in [p]} A_{i,k} B_{k,j}|\leq \max_{i,j,k_0\in[p]}|A_{i,k_0}||\sum_{k\in [p]} B_{k,j}|\leq \ninf{A}\|B\|_{1,\infty}.
\end{align*}
\end{proof}

\section{Proofs of Main Theorems}~\label{sec:proofs}
\subsection{Proof of Theorem \ref{thm:meta}}~\label{sec:proof-meta}
\begin{proof}
Continue from the setup in \Cref{sec:proof-sketch-meta}. We now have to complete the proof by verifying the remaining 4 steps.

\textbf{\large{S}\footnotesize{TEP 1.}}
\normalsize
We begin by finding a probabilistic upper bound on $\|S-\Sigma\|_{\infty,\infty}$. As seen in \cref{eq:sample_cov_pooled},
$S=\frac{1}{m}\sum_{i=1}^m S^{(i)}$, where $S^{(i)}=\frac{1}{n} \sum_{j=1}^{n} X_{j}^{(i)}(X_{j}^{(i)})^T$ are estimated variance matrix from auxiliary tasks.

\begin{align*}
    \ninf{S-\Sigma}&\leq \ninf{\frac{1}{m}\sum_{i=1}^m \lp S^{(i)}-\Sigma^{(i)}\rp}+\ninf{\frac{1}{m}\sum_{i=1}^m \Sigma^{(i)}-\Sigma }\numberthis\label{eq:main-eq1}\\
    & \text{Term 1}+\text{Term 2}.
\end{align*}
\large{T}\footnotesize{ERM 2.}
\normalsize We begin by analysing the second term in \cref{eq:main-eq1}. We note that for any fixed $i$,
\begin{align*}
    \ninf{\Sigma^{(i)}-\Sigma}&=\ninf{R^{(i)}U(\Lambda+D^{(i)})(R^{(i)}U)^T-U\Lambda U^T}\\
    & \leq \ninf{R^{(i)}U(\Lambda+D^{(i)})(R^{(i)}U)^T}+\ninf{U\Lambda U^T}\\
    & = \ninf{(R^{(i)}-\Ibb_p+\Ibb_p)U(\Lambda+D^{(i)})U^T(R^{(i)}-\Ibb_p+\Ibb_p)^T}+\ninf{U\Lambda U^T}\\
    & =\ninf{(R^{(i)}-\Ibb_p)U(\Lambda+D^{(i)})U^T(R^{(i)}-\Ibb_p)^T}+\ninf{U(\Lambda+D^{(i)})U^T(R^{(i)}-\Ibb_p)^T}\\
    &\qquad +\ninf{(R^{(i)}-\Ibb_p)U(\Lambda+D^{(i)})U^T}+\ninf{U\Lambda U^T}.
\end{align*}
Applying \cref{lem:norm-bound} we get,
\begin{align*}
    \ninf{\Sigma^{(i)}-\Sigma}&\leq \|R^{(i)}-\Ibb_p\|_{1,\infty}^2\ninf{U(\Lambda+D^{(i)})U^T}+2\|R^{(i)}-\Ibb_p\|_{1,\infty}\ninf{U(\Lambda+D^{(i)})U^T}\\
    & \qquad +\ninf{U\Lambda U^T},
\end{align*}
which can be upper bounded by,
\begin{align*}
    \lp\|R^{(i)}-\Ibb_p\|_{1,\infty}+1\rp^2\lp\ninf{U\Lambda U^T}+\ninf{UD^{(i)} U^T}\rp.\numberthis~\label{eq:sigma^i-norm}
\end{align*}
Since $U$ is an orthogonal matrix, $\ninf{UD^{(i)}D^T}\leq \lambda_1(D^{(i)})$. Then, from \cref{assume:unbiased}, we get,
\begin{align*}
    \ninf{\Sigma^{(i)}-\Sigma}\leq \lp\constant_r+1\rp^2\lp\lambda_1(\Sigma)+L\rp.
\end{align*}
It is easy to see, by an application of triangle inequality and Jensen's inequality that,
\begin{align*}
    \ninf{\Sigma^{(i)}-\Sigma-\expec[\Sigma^{(i)}-\Sigma]}\leq 2\lp\constant_r+1\rp^2\lp\lambda_1(\Sigma)+L\rp.
\end{align*}
To upper bound $\ninf{\expec[\Sigma^{(i)}-\Sigma]}$, we follow a similar decomposition.
\begin{align*}
    \ninf{\expec[\Sigma^{(i)}-\Sigma]}&=\ninf{\expec[R^{(i)}U(\Lambda+D^{(i)})(R^{(i)}U)^T]-U\Lambda U^T}\\
    & = \ninf{\expec[(R^{(i)}-\Ibb_p+\Ibb_p)U(\Lambda+D^{(i)})U^T(R^{(i)}-\Ibb_p+\Ibb_p)^T]-U\Lambda U^T}
\intertext{Since $D^{(i)}$ is independent of $R^{(i)}$, we can decompose the expectation of products into the product of expectations. Consequently, the terms $D^{(i)}$ vanishes, and we are left with the following equality,}
    \ninf{\expec[\Sigma^{(i)}-\Sigma]} & =\ninf{\expec[(R^{(i)}-\Ibb_p+\Ibb_p)U\Lambda U^T(R^{(i)}-\Ibb_p+\Ibb_p)^T]-U\Lambda U^T}\\
    & = \ninf{\expec\lb(R^{(i)}-\Ibb_p)U\Lambda U^T(R^{(i)}-\Ibb_p)^T+U\Lambda U^T(R^{(i)}-\Ibb_p)^T+(R^{(i)}-\Ibb_p)U\Lambda U^T\rb}.
\end{align*}
Repeatedly applying \cref{lem:norm-bound} and Jensen's inequality, we get
\begin{align*}
    \ninf{\expec[\Sigma^{(i)}-\Sigma]} & \leq \expec[\|R^{(i)}-\Ibb_p\|_{1,\infty}]^2\ninf{U\Lambda U^T}+2\expec[\|R^{(i)}-\Ibb_p\|_{1,\infty}]\ninf{U\Lambda U^T},
 \intertext{which, upon applying Chebyshev's inequality on $\expec[\|R^{(i)}-\Ibb_p\|_{1,\infty}]$ yields,} 
 \ninf{\expec[\Sigma^{(i)}-\Sigma]} & \leq \frac{\constant_R}{p}\lambda_1(\Sigma)+2\sqrt{\frac{\constant_R}{p}}\lambda_1(\Sigma).\numberthis\label{eq:expec-error-bound}
\end{align*}
Now, putting $A_i=\Sigma^{(i)}-\Sigma-\expec[\Sigma^{(i)}-\Sigma]$ applying \cref{lemma:matrix-hoeff}, we get that for any $\epsilon>0$
\begin{align*}
    \prob\lp\frac{1}{m}\sum_{i=1}^m\ninf{\Sigma^{(i)}-\Sigma-\expec[\Sigma^{(i)}-\Sigma]}>\epsilon\rp\leq 2pe^{-\frac{3\epsilon^2m}{8\lp\lambda^{\dagger}\rp^2}}.
\end{align*}
Consequently, from \cref{eq:expec-error-bound} and an application of triangle inequality we get,
\begin{align*}
     \prob\lp\frac{1}{m}\sum_{i=1}^m\ninf{\Sigma^{(i)}-\Sigma}>\epsilon\rp\leq 2pe^{-\frac{3m\lp\epsilon-\frac{\constant_R}{p}\lambda_1(\Sigma)-2\sqrt{\frac{\constant_R}{p}}\lambda_1(\Sigma)\rp^2}{8\lp\lambda^{\dagger}\rp^2}}
\end{align*}

Let \begin{align*}
    \epsilon = 4\lambda^{\dagger}\sqrt{\frac{\log(p)}{m}}+\frac{\constant_R}{p}\lambda_1(\Sigma)+2\sqrt{\frac{\constant_R}{p}}\lambda_1(\Sigma).
\end{align*}
Then, with probability at most $\frac{2}{p^2}$,
\begin{align*}
    \frac{1}{m}\sum_{i=1}^m\ninf{\Sigma^{(i)}-\Sigma}\geq 4\lambda^{\dagger}\sqrt{\frac{\log(p)}{m}}+\frac{\constant_R}{p}\lambda_1(\Sigma)+2\sqrt{\frac{\constant_R}{p}}\lambda_1(\Sigma).
\end{align*}

\large{T}\footnotesize{ERM 1.} \normalsize In this part of the proof, we provide an upper bound for the first term from \cref{eq:main-eq1}. We begin by finding an upper bound to $\ninf{\Sigma^{(i)}}$. Similarly, as to the derivation of \cref{eq:sigma^i-norm} we can also write,
\begin{align*}
    \ninf{\Sigma^{(i)}}\leq \lambda^{\dagger}.
\end{align*}
By using \Cref{lem:TailSampleCov} we get,
\begin{align*}
    \prob\lp \ninf{\frac{1}{m}\sum \lp S^{(i)}-\Sigma^{(i)}\rp} \geq \epsilon \rp\leq \frac{p(p+1)}{2}e^{-\frac{nm\epsilon^2}{512\sigma^4\lambda^{\dagger}}}.
\end{align*}
putting $\epsilon=16\sqrt{2\sigma^4\lambda^{\dagger} \frac{\log\lp p+1 \rp}{nm}}$, we get with probability $1-\frac{1}{2(p+1)^2}$, 
\begin{align*}
    \prob\lp \ninf{\frac{1}{m}\sum \lp S^{(i)}-\Sigma^{(i)}\rp} \leq 16\sqrt{2\sigma^4\lambda^{\dagger} \frac{\log\lp p+1 \rp}{nm}} \rp.
\end{align*}
Thus, 
\begin{align*}
    &\prob\lp \ninf{S-\Sigma} > \max\lc 4\lambda^{\dagger}\sqrt{\frac{\log(p)}{m}}+\frac{\constant_R}{p}\lambda_1(\Sigma)+2\sqrt{\frac{\constant_R}{p}}\lambda_1(\Sigma), 16\sqrt{2\sigma^4\lambda^{\dagger} \frac{\log\lp p+1 \rp}{nm}}\rc\rp\\
    &\qquad \leq \frac{2}{p^2}+\frac{1}{2(p+1)^2}
\end{align*}
Since, 
\[
\frac{8|J|}{\lambda_k-\lambda_{k+1}}\lVert\Sigma_{J^c, J}\rVert _{2,\infty}<1,
\]
there exists a constant $\alpha\in (0,1)$ such that if,
\[
\rho>\alpha\max\lc4\lambda^{\dagger}\sqrt{\frac{\log(p)}{m}}+\frac{\constant_R}{p}\lambda_1(\Sigma)+2\sqrt{\frac{\constant_R}{p}}\lambda_1(\Sigma), 16\sqrt{2\sigma^4\lambda^{\dagger} \frac{\log\lp p+1 \rp}{nm}}\rc,
\]
then with probability at least $1-\frac{2}{p^2}-\frac{1}{2(p+1)^2}$,
\begin{align}
    \rho^{-1}\lVert S-\Sigma\rVert_{\infty,\infty}+
\frac{8|J|}{\lambda_{diff}}\lVert\Sigma_{J^c, J}\rVert _{2,\infty}\leq1.
\end{align}
It only remains to choose an appropriate matrix $Q$. By the application of \Cref{lem:sparse-solution} there exists a $Q$ matrix such that 
\[
\lVert Q-\Ibb_{|J|} \rVert_F \le\frac{8\rho |J|}{\lambda_{diff}}, 
\]
and
\[
U_J=Q\hat U_J,
\]
where $U_J$ and $\hat U_J$ are defined as in \Cref{lem:sparse-solution}. We select this $Q$ for the remainder of the proof.

\textbf{\large{S}\footnotesize{TEP 2. }}\normalsize
 In this section of the proof, we show that our proposed estimator is feasible, unique, and optimal, i.e., satisfies the KKT conditions (\ref{eq:kkt-1},\ref{eq:kkt-2},\ref{eq:kkt-3}).

\large{F}\footnotesize{EASIBLITY.} \normalsize The feasibility of $\hat H$ is obvious from construction. To check the feasibility of $\hat Z$, we need to verify that for a given choice of $\rho$, 
\[
\hat{Z}_{ij}\in [-1,1] \forall  (i,j)\in (J\times
J)^c.
\]
Therefore, it is sufficient to verify
\begin{align*}
 \frac{1}{\rho} \bigl[ |S_{ij}-
\Sigma_{ij}|+\bigl|\Sigma_{ij}-\langle Q_{i*},
\Sigma_{J,j} \rangle\bigr| \bigr]\leq 1.
\end{align*}
By an application of Cauchy-Schwarz inequality, we get
\begin{align*}
     \frac{1}{\rho} \bigl[ |S_{ij}-
\Sigma_{ij}|+\bigl|\Sigma_{ij}-\langle Q_{i*},
\Sigma_{J,j} \rangle\bigr| \bigr] & \leq  \frac{1}{\rho} \bigl[ |\ninf{S-
\Sigma}+\|Q-\Ibb_{|J|}\|_{F}\|\Sigma\|_{2,\infty}\bigr]\\
& \leq \frac{1}{\rho} \lb \ninf{S-
\Sigma}+\frac{8\rho |J|}{\lambda_{diff}}\|\Sigma\|_{2,\infty}\rb.
\end{align*}
Thus, with probability at least $1-\frac{2}{p^2}-\frac{1}{2(p+1)^2}$,
\begin{align}
    \rho^{-1}\lVert S-\Sigma\rVert_{\infty,\infty}+
\frac{8|J|}{\lambda_{diff}}\lVert\Sigma_{J^c, J}\rVert _{2,\infty}\leq 1.
\end{align}
\large{O}\footnotesize{PTIMALITY.} \normalsize $\hat H$ only has non-zero elements in the sub-matrix $\hat H_{J,J}$. By construction $(\hat H_{i,j},\hat Z_{i,j})=(\tilde H_{i,j},\tilde Z_{i,j})\forall (i,j)\in J\times J$. From the optimality of $(\tilde H, \tilde Z)$, \Cref{eq:kkt-1} is satisfied.

The fact that \Cref{eq:kkt-2} is satisfied follows from the optimality of $(\tilde H, \tilde Z)$ when $(i,j)\in J\times J$ and feasibility of $\tilde Z$ otherwise. 

Note that $\hat U_J$ spans the $k$-dimensional principal subspace of $S-\rho \tilde Z$. We now show that $\hat U$ spans one (not necessarily principal) $k$-dimensional subspace of $S-\rho\hat Z$. In other words, it is enough to show that, for some diagonal matrix $\initialD$,
\[
(S-\rho\hat Z)\begin{bmatrix}
        \hat U_J\\
        0
    \end{bmatrix}= D\begin{bmatrix}
        \hat U_J\\
        0
    \end{bmatrix}.~\numberthis\label{eq:subspace}
\]
Analysing the term on the left-hand side of the previous equation, we get,
\begin{align*}
    (S-\rho\widehat{Z})\begin{bmatrix}
        \hat U_J\\
        0
    \end{bmatrix} & = 
    \begin{bmatrix}
            S_{J,J}-
        \rho\tilde Z_{J,J} & Q \Sigma_{J,J^c}
        \\
        \Sigma_{J^c, J}Q^T
        & \Sigma_{J^c, J^c}+\Dbb(S_{J^c, J^c}-\Sigma_{J^c, J^c})
    \end{bmatrix}\begin{bmatrix}
        \hat U_J\\
        0
\end{bmatrix}\\
& = \begin{bmatrix}
  (S_{J,J} -\rho \tilde Z_{J,J})\hat U_{J}\\
  \Sigma_{J^c J}Q^T \hat U_J
\end{bmatrix}\\
& = \begin{bmatrix}
  (S_{J,J} -\rho \tilde Z_{J,J})\hat U_{J}\\
  \Sigma_{J^c J}Q^T Q U_J
\end{bmatrix}
\end{align*}
Recall from \Cref{lem:sparse-solution} that there exists a diagonal matrix $D_1$ such that $(S_{J,J} -\rho \tilde Z_{J,J})\hat U_{J}=D_1\hat U_J$. Also, the fact that $\begin{bmatrix}
        U_J\\
        0
\end{bmatrix}$ spans $k$-dimensional principal subspace of $\Sigma$ implies that $\Sigma_{J^c J} U_{J}=0$. Therefore,
\begin{align*}
    \begin{bmatrix}
  (S_{J,J} -\rho \tilde Z_{J,J})\hat U_{J}\\
  \Sigma_{J^c J}Q^T Q U_J
\end{bmatrix} & = \begin{bmatrix}
  D_1\hat U_{J}\\
  \Sigma_{J^c J} U_J
\end{bmatrix}\\
& = \begin{bmatrix}
  D_1\hat U_{J}\\
  0
\end{bmatrix}\\
& =\begin{bmatrix}
  D_1\hat U_{J}\\
  0
\end{bmatrix}\\
& = \begin{bmatrix}
    D_1 & 0\\
    0 & 0
\end{bmatrix}
\begin{bmatrix}
    \hat U_J\\
    0
\end{bmatrix}
\end{align*}
Setting
\[
D_0= \begin{bmatrix}
     D_1 & 0\\
    0 & 0
\end{bmatrix}
\]
we successfully show that \Cref{eq:subspace} holds. We will further prove that $diag(D_1)$ are the first $k$ eigenvalues of $(S-\rho\widehat{Z})$, and 
$\lambda_k(S-\rho\widehat{Z})-\lambda_{k+1}(S-\rho\widehat{Z})>0$. Then uniqueness will follow by \cref{lem:sparse-solution}.
\begin{align*}
(S-\rho\hat Z) & = \begin{bmatrix}
    S_{JJ}-\rho\tilde
Z_{JJ} - Q\Sigma_{JJ} Q^T & 0
\cr
0&\Sigma_{J^c, J^c}+\Dbb(S_{J^c, J^c}-\Sigma_{J^c, J^c})
\end{bmatrix}  + \begin{bmatrix}
    Q\Sigma_{JJ} Q^T & Q\Sigma_{JJ^c}
\vspace*{3pt}\cr
\Sigma_{J^cJ}Q^T & \Sigma_{J^cJ^c}
\end{bmatrix}\\
& =\begin{bmatrix}
    S_{JJ}-\rho\tilde
Z_{JJ} - Q\Sigma_{JJ} Q^T & 0
\cr
0&\Sigma_{J^c, J^c}+\Dbb(S_{J^c, J^c}-\Sigma_{J^c, J^c})
\end{bmatrix} + \begin{bmatrix}
 Q & 0
\cr
0 & I   
\end{bmatrix}
  \times \Sigma \times
\begin{bmatrix}
 Q^T & 0
\cr
0 & I   
\end{bmatrix}
\end{align*}
However, since we just proved that $(S-\rho\widehat{Z})$ satisfies the conditions of  \Cref{prop:sub_eigen}, the problem reduces to finding an upper bound to  $\lambda_l( S_{J,J}-\rho\tilde
Z_{J,J} - Q\Sigma_{J,J}Q^T)$ for each $l\in[|J|]$. Note that,
\begin{align*}
    \lambda_l( S_{J,J}-\rho\tilde
Z_{J,J} - Q\Sigma_{J,J}Q^T) & \leq \lambda_l( S_{J,J}-\rho\tilde
Z_{J,J}-\Sigma_{J,J}) + \lambda_l(\Sigma_{J,J} -  Q\Sigma_{J,J}Q^T)\\
& \leq \lVert S_{J,J}-\rho\tilde
Z_{J,J}-\Sigma_{J,J}\rVert_F+\lVert \Sigma_{J,J} -  Q\Sigma_{J,J}Q^T \rVert_F\\
& \leq  \lVert S_{J,J}-\rho\tilde
Z_{J,J}-\Sigma_{J,J}\rVert_F+2\lVert\Sigma_{J,J}\rVert_F \times\lVert
Q-I\rVert_F \\
& \leq 2\rho s + 2\lambda_1(\Sigma)\times8\rho |J|/(\lambda_k-
\lambda_{k+1}).
\end{align*}
Where the last inequality follows from the combination of part 1 and \Cref{lem:sparse-solution}. The rest of the proof since for a large enough $\rho$ we have,
\begin{align*}
    4\rho |J|+\frac{32\lambda_1(\Sigma)\rho |J|}{(
\lambda_k-\lambda_{k+1})}\le \lambda_k-
\lambda_{k+1}.
\end{align*}
This completes the proof of step 2.

\textbf{\large{S}\footnotesize{TEP 3. }}\normalsize In this section of the proof, we prove that $\ninf{\hat H - \Pi}<2\frac{\rho|J|}{\lambda_{diff}}$. By step 2, we have,
\begin{align*}
    \ninf{\hat H - \Pi} & = \ninf{ \begin{bmatrix}
        \hat U_J \hat U_J^T & 0\\
        0 & 0
    \end{bmatrix}-\begin{bmatrix}
         U_J U_J^T & 0\\
        0 & 0
    \end{bmatrix}}\\
    & = \ninf{ \begin{bmatrix}
        Q U_J U_J^T Q^T& 0\\
        0 & 0
    \end{bmatrix}-\begin{bmatrix}
         U_J U_J^T & 0\\
        0 & 0
    \end{bmatrix}}\\
    & = \ninf{Q U_J U_J^T Q^T -U_J U_J^T}\\
    & \leq \lambda_1(U_J U_J^T) \|\Ibb - Q\|_F\\
    & \leq \frac{8|J|\rho}{\lambda_{diff}}.
\end{align*}
Setting $\constant_{\pi}=\frac{8|J|}{\lambda_{diff}}$, completes the proof of the upper bound for  $\ninf{\hat H - \Pi}$.

\textbf{\large{S}\footnotesize{TEP 4. }}\normalsize 
Set
\[
\rho<\min_{i\in J}\frac{\Pi_{i,i}\lambda_{diff}}{16|J|}.
\]
The rest of the proof now follows from the step 3.
\end{proof}
\subsection{Proof of Proposition \ref{prop:sub_eigen}}\label{sec:proof-sub-eigen}
\begin{proof}
Let $\Pi^{(m+1)}$ be the $k$-th dimensional principal component corresponding to $\Sigma^{(m+1)}$. Hence,
\begin{align*}
    \Pi^{(m+1)}=\sum_{i=1}^k U_{*,i}^{(m+1)}\lp U_{*,i}^{(m+1)}\rp^T.
\end{align*}
We also know from the recovered support union that $\Pi_{i,i}^{(m+1)}=0$ whenever $i\notin J$. However,
\begin{align*}
    \Pi_{i,i}^{(m+1)} = \sum_{j=1}^k \lp U_{j,i}^{(m+1)}\rp^2.
\end{align*}
Hence, $\forall\enspace j\in J^c$, $u_{i,j}^{(m+1)}=0$. Therefore, \[
U_{i}^{(m+1)}=\begin{pmatrix}
U_{J,i}^{(m+1)}\\
0\end{pmatrix}.\]
Moreover, we also know that 
\begin{align*}
   \Sigma^{(m+1)} U_{i}^{(m+1)} = \lambda_i U_{i}^{(m+1)}.
\end{align*}
In other words,
\begin{align*}
     \begin{bmatrix}
    \Sigma_{J,J}^{(m+1)} & \Sigma_{J,J^c}^{(m+1)}\\
    \Sigma_{J^c,J}^{(m+1)} & \Sigma_{J^c,J^c}^{(m+1)}
    \end{bmatrix} \begin{pmatrix}
U_{J,i}^{(m+1)}\\
0\end{pmatrix} = \lambda_i \begin{pmatrix}
U_{J,i}^{(m+1)}\\
0\end{pmatrix}.
\end{align*}
Which in turn implies, 
\begin{align*}
     \Sigma_{J,J}^{(m+1)}U_{J,i}^{(m+1)}=\lambda_i U_{J,i}^{(m+1)}.
\end{align*}
This completes the proof.
\end{proof}

\subsection{Proof of Theorem \ref{thm:novel}}\label{sec:proof-novel}
\begin{proof}
As given, $\Sigma^{(m+1)}$ is the novel covariance matrix with the following block structure,
\begin{align*}
   \Sigma^{(m+1)} = \begin{bmatrix}
    \Sigma_{J,J}^{(m+1)} & \Sigma_{J,J^c}^{(m+1)}\\
    \Sigma_{J^c,J}^{(m+1)} & \Sigma_{J^c,J^c}^{(m+1)}
    \end{bmatrix}.
\end{align*}
For $1\leq i\leq k$, let $U_i^{(m+1)}$ be the eigenvectors corresponding to $\Sigma^{(m+1)}$, and $\lambda_i^{(m+1)}$ be the corresponding eigenvalues. Let $J$ be the recovered support union. $U_{J,i}^{(m+1)}$ denotes those coordinates of $U_{i}^{(m+1)}$ that belong to set $J$. Similarly to the proof of \Cref{thm:meta}, if the penalty parameter $\rho$ satisfies,
\begin{align*}
    & \rho^{-1}\lVert S_{J,J}^{(m+1)}-\Sigma_{J,J}^{(m+1)}\rVert_{\infty,\infty}+
\frac{8|J|^{(m+1)}}{\lambda_k^{(m+1)}-\lambda_{k+1}^{(m+1)}}\lVert\Sigma_{(J^{(m+1})^c J^{(m+1)}}\rVert _{2,\infty}\le1 \quad \text{and }\\
& 4\rho |J^{(m+1)}| \biggl(1+
\frac{8\lambda
_1^{(m+1)}}{\lambda_k^{(m+1)}-\lambda_{k+1}^{(m+1)}} \biggr)<\lambda_k^{(m+1)}-\lambda_{k+1}^{(m+1)},
\end{align*}
then the solution $\widehat{H}^{(m+1)}$ of \Cref{eq:novel_objective} is
unique and
satisfies\break ${supp}(\widehat
{H}^{(m+1)})\subseteq J^{(m+1)}$. It is noteworthy that if for each $j\in[n^{(m+1)}]$, $X^{(m+1)}_j$ is a $p$-dimensional multivariate sub-Gaussian distribution with covariance matrix $\Sigma$ and parameter $\sigma$, then for any set of sub-indices $J$, $X_{j,J}^{(m+1)}$ is a $|J|$-dimensional multivariate sub-Gaussian distribution with covariance matrix $\Sigma_{J,J}$ and parameter $\sigma$. Following from \cref{eq:block-cov-est}, we can also see that by construction
\[
S_{J,J}^{(m+1)}=\frac{1}{n^{(m+1)}}\sum_{l=1}^{n^{(m+1)}} X_{l,J}^{(m+1)}\lp X_{l,J}^{(m+1)}\rp^T.
\]
By using \Cref{lem:TailSampleCov}, we know that as long as $\epsilon\leq 32\lambda_1\lp\Sigma^{(m+1)}\rp\sigma^2$,
\begin{align*}
 \prob\lp \ninf{ \lp S^{(m+1)}-\Sigma^{(m+1)}\rp} \geq \epsilon \rp\leq \frac{|J|(|J|+1)}{2}e^{-\frac{n\epsilon^2}{512\sigma^4\lambda_1\lp\Sigma^{(m+1)}\rp}}.    
\end{align*}

Define $\alpha^{(m+1)}:=1-\frac{8|J|^{(m+1)}}{\lambda_k^{(m+1)}-\lambda_{k+1}^{(m+1)}}\lVert\Sigma_{(J^{(m+1})^c J^{(m+1)}}\rVert _{2,\infty}$. By \Cref{assume:novel-corr-cond}, $\alpha^{(m+1)}\in (0,1)$. Therefore, as long as, 
\begin{align*}
    &n^{\frac{1}{3}}>\max\lc\frac{1}{32\lambda_1\lp\Sigma^{(m+1)}\sigma^2\rp}, \lp\frac{\lambda_k^{(m+1)}-\lambda_{k+1}^{(m+1)}}{4|J^{(m+1)}|\lp 1+
\frac{8\lambda
_1^{(m+1)}}{\lambda_k^{(m+1)}-\lambda_{k+1}^{(m+1)}}\rp}\rp^{-1}\rc \text{ and,}\\
    &\frac{1}{ n^{1/3}\alpha^{(m+1)}}<\rho< \frac{\lambda_k^{(m+1)}-\lambda_{k+1}^{(m+1)}}{4|J^{(m+1)}|\lp 1+\frac{8\lambda
_1^{(m+1)}}{\lambda_k^{(m+1)}-\lambda_{k+1}^{(m+1)}}\rp},
\end{align*}
with probability $1-\frac{|J|(|J|+1)}{2}e^{-\frac{n^{\frac{1}{3}}}{512\sigma^4\lambda_1\lp\Sigma^{(m+1)}\rp}}$, we have $supp(\Pi^{(m+1)})= J^{(m+1)}$.
\end{proof}
\subsection{Proof of Lemma \ref{lem:TailSampleCov}}~\label{sec:TailSampleCov}
\begin{proof}
We note that $S^{(i)}-\Sigma^{(i)}$ are i.i.d with $\expec\lb S^{(i)}-\Sigma^{(i)}\rb=0$. We begin by finding an entrywise bound on the elements of $\frac{1}{m}\sum \lp S^{(i)}-\Sigma^{(i)}\rp$. Since $S^{(i)}=\frac{1}{p}\sum X_p^{(i)}\lp X_p^{(i)}\rp^T$,
\begin{align*}
    \prob\lp \left|\frac{1}{mn}\sum_{i,p}X_{k,p}^{(i)}X_{l,p}^{(i)}-\sigma_{k,l}^{(i)}\right|\geq \epsilon \rp.
\end{align*}
Let $\sigma_k:=\max_{1\le i\le K}{\sigma}_{k,k}^{(i)}$,
$\tilde{X}_{k,p}^{(i)}:=\frac{X_{k,p}^{(i)}}{\sqrt{\sigma_k}}$,
$\tilde{\rho}_{ij}^{(k)}:=\frac{{\sigma}_{k,l}^{(i)}}{\sqrt{\sigma_k\sigma_l}}$. We have
\begin{equation*}
\prob\lp \left|\frac{1}{mn}\sum_{i,p}X_{k,p}^{(i)}X_{l,p}^{(i)}-\sigma_{k,l}^{(i)}\right|\geq \epsilon \rp =\mathbb{P}\left[4
\left|\sum_{i,p}\left(\tilde{X}_{k,p}^{(i)}\tilde{X}_{l,p}^{(i)}-\tilde{\rho}_{k,l}^{(i)}\right)\right|
>\frac{4nm\epsilon}{\sqrt{\sigma_k\sigma_l}}\right]
\end{equation*}
Define $U^{(i)}_{p,k,l}:=\tilde{X}_{p,k}^{(i)}+\tilde{X}_{p,l}^{(i)}$, $V^{(i)}_{p,k,l}:=\tilde{X}_{p,k}^{(i)}-\tilde{X}_{p,l}^{(i)}$. 
Then for any $r\in\mathbb{R}$,
\begin{equation}
4\sum_{i,p}\left(\tilde{X}_{k,p}^{(i)}\tilde{X}_{l,p}^{(i)}-\tilde{\rho}_{k,l}^{(i)}\right)=
\sum_{k,t}\left\{\left(U^{(i)}_{p,k,l}\right)^2-2\left(r+\tilde{\rho}^{(i)}_{k,l}\right)\right\}-
\sum_{k,t}\left\{\left(V^{(i)}_{p,k,l}\right)^2-2\left(r-\tilde{\rho}^{(i)}_{k,l}\right)\right\}
\end{equation}
Thus,
\begin{equation} \label{eq:Ineqstar1}
\begin{aligned}
 \prob\lp \left|\frac{1}{mn}\sum_{i,p}X_{k,p}^{(i)}X_{l,p}^{(i)}-\sigma_{k,l}^{(i)}\right|\geq \epsilon \rp\leq & \mathbb{P}\left[\Big|\sum_{k,t}
\left\{\left(U^{(k)}_{t,ij}\right)^2-2\left(r+\tilde{\rho}^{(k)}_{ij}\right)\right\}\Big|
>\frac{2nm\epsilon}{\sqrt{\sigma_k\sigma_l}}\right] \\&+ \mathbb{P}\left[\Big|\sum_{k,t}
\left\{\left(V^{(k)}_{t,ij}\right)^2-2\left(r-\tilde{\rho}^{(k)}_{ij}\right)\right\}\Big|
>\frac{2nm\epsilon}{\sqrt{\sigma_k\sigma_l}}\right].
\end{aligned}
\end{equation}

Consider the first term 

\begin{align*}
    \mathbb{P}\left[\Big|\sum_{i,p}
\left\{\left(U^{(i)}_{p,k,l}\right)^2-2\left(r+\tilde{\rho}^{(k)}_{ij}\right)\right\}\Big|
>\frac{2nm\epsilon}{\sqrt{\sigma_k\sigma_l}}\right],
\end{align*}

and let $r=\frac{\sigma_{k,k}^{(i)}}{\sigma_k}+\frac{\sigma_{l,l}^{(i)}}{\sigma_l}$. Then, 

\[
\expec\lb \left(U^{(i)}_{p,k,l}\right)^2-2\left(r+\tilde{\rho}^{(k)}_{ij}\right) \rb=0.
\]

Let $Z_{p,k,l}^{(i)}:=\left(U^{(i)}_{p,k,l}\right)^2-2\left(r+\tilde{\rho}^{(k)}_{ij}\right)$. Consequently, 

\begin{align*}
    \prob\lp \Big|\sum_{i,p}
\left\{\left(U^{(i)}_{p,k,l}\right)^2-2\left(r+\tilde{\rho}^{(k)}_{ij}\right)\right\}\Big|>\frac{2nm\epsilon}{\sqrt{\sigma_k\sigma_l}}\rp & = \prob \lp \sum_{i,p}
Z_{p,k,l}^{(i)} > \frac{2nm\epsilon}{\sqrt{\sigma_k\sigma_l}} \rp\numberthis~\label{eq:sub-exp}\\
& \qquad + \prob \lp \sum_{i,p}
Z_{p,k,l}^{(i)}< -\frac{2nm\epsilon}{\sqrt{\sigma_k\sigma_l}} \rp
\end{align*}

To see that $U_{p,k,l}^{(i)}$ is sub-Gaussian, we note that,
\begin{align*}
   \expec(e^{\lambda U_{p,k,l}^{(i)}}) & = \expec\lp e^{ \lambda \lp X_{p,k}^{(i)}+X_{p,l}^{(i)}\rp}\rp \\
   & \leq \lb \expec e^{2\lambda \tilde{X}_{p,k}^{(i)}}\rb^{0.5}\lb \expec e^{2\lambda \tilde{X}_{p,l}^{(i)}}\rb^{0.5}\\
   & \leq  e^{\lambda^2 \lp\frac{\sigma_{k,k}^{(i)}}{\sigma_k}\sigma\rp^2 }e^{\lambda^2 \lp\frac{\sigma_{l,l}^{(i)}}{\sigma_l}\sigma\rp^2 }.
\intertext{Since $\frac{\sigma_{k,k}^{(i)}}{\sigma_k}$ and $\frac{\sigma_{l,l}^{(i)}}{\sigma_l}$ are upper bounded by 1, we consequently have, }
   \expec(e^{\lambda U_{p,k,l}^{(i)}}) & \leq e^{2\lambda^2\sigma^2},
\end{align*}
which shows that $U_{p,k,l}^{(i)}$ is sub-Gaussian with parameter $2\sigma$. We also know, from \cite[Section 5.2.4]{vershynin2010introduction}, that if $U_{p,k,l}^{(i)}$ is a sub-Gaussian random variable with then $\lp U_{p,k,l}^{(i)}\rp^2$ is a sub-exponential random variable. As noted in \Cref{sec:preliminaries} if $X$ is a sub-Gaussian with parameter $\omega$, then, $X^2$ is a sub-exponential distribution with parameters $(4\sqrt{2}\omega^2,4\omega^2)$. Consequently, if $|t|<\frac{1}{4\omega^2}$,
\begin{align*}
    \expec[e^{tX^2}]\leq e^{16\omega^4t^2}.
\end{align*}
Returning to \cref{eq:sub-exp}, we observe that each $Z_{p,k,l}^{(i)}$ is i.i.d. subexponential with parameter $(16\sqrt{2}\sigma^2,16\sigma^2)$. Hence for any $\frac{1}{16\sigma^2}>\theta>0$,
\begin{align*}
    \prob \lp \sum_{i,p}
Z_{p,k,l}^{(i)} > \frac{2nm\epsilon}{\sqrt{\sigma_k\sigma_l}} \rp & = \prob \lp \theta\sum_{i,p}
Z_{p,k,l}^{(i)} > \theta\frac{2nm\epsilon}{\sqrt{\sigma_k\sigma_l}} \rp\\
& =  \prob \lp e^{\theta\sum_{i,p}
Z_{p,k,l}^{(i)}} > e^{\theta\frac{2nm\epsilon}{\sqrt{\sigma_k\sigma_l}}} \rp.
\end{align*}
Thus, by Markov's inequality, it follows that, 
\begin{align*}
       \prob \lp \sum_{i,p}
Z_{p,k,l}^{(i)} > \frac{2nm\epsilon}{\sqrt{\sigma_k\sigma_l}} \rp & \leq \expec\lb e^{\theta\sum_{i,p}
Z_{p,k,l}^{(i)}}\rb e^{-\theta\frac{2nm\epsilon}{\sqrt{\sigma_k\sigma_l}}}\\
& = \prod_{i,p}\expec\lb e^{\theta Z_{p,k,l}^{(i)}} \rb e^{-\theta\frac{2nm\epsilon}{\sqrt{\sigma_k\sigma_l}}}\\
& \leq \prod_{i,p} e^{512\theta^2\sigma^4}e^{-\theta\frac{2nm\epsilon}{\sqrt{\sigma_k\sigma_l}}}\\
& = e^{512nm\theta^2\sigma^4}e^{-\theta\frac{2nm\epsilon}{\sqrt{\sigma_k\sigma_l}}}\\
& = e^{512mn\theta^2\sigma^4-\theta\frac{2nm\epsilon}{\sqrt{\sigma_k\sigma_l}}}
\end{align*}
Minimising over $\theta>0$, we get the derivative of the exponent with respect to $\theta$ to be equal to $0$. In other words,
\begin{align*}
    512mn\theta\sigma^4-\frac{2nm\epsilon}{\sqrt{\sigma_k\sigma_l}}=0,
\end{align*}
which in turn implies that whenever
\[
\frac{\epsilon}{512\sigma^4\sqrt{\sigma_k\sigma_l}}<\frac{1}{16\sigma^2},
\]
the minimum is achieved when $\theta=\frac{\epsilon}{512\sigma^4\sqrt{\sigma_k\sigma_l}}$.
In other words, if 
\[
\epsilon<32\sigma^2\sqrt{\sigma_k\sigma_l},
\]
the minimum is, $e^{-\frac{nm\epsilon^2}{512\sigma^4\sqrt{\sigma_k\sigma_l}}}$. Consequently, our upper bound can be further refined as,
\begin{align*}
      \prob \lp \sum_{i,p}
Z_{p,k,l}^{(i)} > \frac{2nm\epsilon}{\sqrt{\sigma_k\sigma_l}} \rp \leq e^{-\frac{nm\epsilon^2}{512\sqrt{\sigma_k\sigma_l}\sigma^4}}.
\end{align*}
Since for each $k\in[m]$, $\sigma_k<\max_{i\in[m]}\ninf{\Sigma^{(i)}}<\eta$, this in turn implies that the upper bound is independent of the value of $\Sigma^{(i)}$'s. In particular, this upper bound is distribution free. Hence, for $\epsilon<32\eta\sigma^2$,
\[
 \prob \lp \sum_{i,p}
Z_{p,k,l}^{(i)} > \frac{2nm\epsilon}{\sqrt{\sigma_k\sigma_l}} \rp \leq e^{-\frac{nm\epsilon^2}{512\eta\sigma^4}}.
\]
Similarly, the upper bound to $\prob \lp \sum_{i,p}
Z_{p,k,l}^{(i)}< -\frac{2nm\epsilon}{\sqrt{\sigma_k\sigma_l}} \rp$, can be achieved; and,
\begin{align*}
    \prob \lp \sum_{i,p}
Z_{p,k,l}^{(i)}< -\frac{2nm\epsilon}{\sqrt{\sigma_k\sigma_l}} \rp \leq e^{-\frac{nm\epsilon^2}{512\eta\sigma^4}}.
\end{align*}
Hence, we have found the entrywise error bound to $\frac{1}{m}\sum \lp S^{(i)}-\Sigma^{(i)}\rp$. To find an upper bound to  $ \ninf{\frac{1}{m}\sum \lp S^{(i)}-\Sigma^{(i)}\rp}$, we use a union bound.
\begin{align*}
     \prob\lp \ninf{\frac{1}{m}\sum \lp S^{(i)}-\Sigma^{(i)}\rp} > \epsilon \rp& \leq \sum_{k,l} \prob\lp \left|\frac{1}{mn}\sum_{i,p}X_{k,p}^{(i)}X_{l,p}^{(i)}-\sigma_{k,l}^{(i)}\right|\geq \epsilon \rp.\\
     & \leq \frac{p(p+1)}{2}e^{-\frac{nm\epsilon^2}{512\eta\sigma^4}}.
\end{align*}
This completes the proof.
\end{proof}
\section{Additional Experiments}\label{sec:additional-exp}
To validate our results, we first generate random matrices $\Sigma^{(i)}$ for each $i\in [m]$ and then generate $X^{(i)}_{j}$ for each $j\in [n]$.  
We now describe the generative procedure for $\Sigma^{(i)}$. We first produce a random orthogonal matrix $U$ where $U_{l,*}$ is supported on $[J]$ for each $l\in [k]$. We then produce a random diagonal matrix with each entry produced uniformly from $[0,1]$ and add $500$ to the first $k$ many diagonal elements to produce $\Lambda$.
$
\Sigma=U\Lambda(U)^T.
$
For each $i\in [m]$, we generate randomised uniform white noise matrices $\Acal^{(i)}$ $\Acal^{(i)}_{l,j}\overset{iid}{\sim} Uniform(0,1)$. $R^{(i)}$ is produced by adding scaled $\Acal^{(i)}$'s to $\Ibb_p$; i.e. 
$
R^{(i)} = \Acal^{(i)}/p^2+\Ibb_p.
$
We generate $D^{(i)}$'s, as white noise matrices where for each $l,j\in [p]; l>j, D^{(i)}_{l,j}$ is generated uniformly from $(0,1)$. This ensures \Cref{assume:unbiased} is satisfied. $\Sigma^{i}$ is then generated as,
$
\Sigma^{(i)}=R^{(i)}U(\Lambda+D^{(i)})U^T\lp R^{(i)}\rp^T.
$ We compute the empirical probability of successful support recovery $\prob(\hat{J}=J)$ as the number of times we obtain the exact support recovery among 100 repetitions, divided by 100. We compute the standard deviation as $({\prob(\hat{J}=J)(1-\prob( \hat{J}=J))})/{100}$. Performance is measured by the empirical probability of correct classification $\prob(\hat{J}=J)$ and maximal estimation error $\|\hat{\Pi}-\Pi\|_{\infty,\infty}$. To create error bars we use a empirical 95\% confidence interval around the estimate.

\paragraph{Uniform Distribution.} For all the experiments in this sub-section, we let $p=80$, $k=6$ and $|J|=6$ and perform 100 repetitions of each setting. We now consider the setting of different dimensions. We choose $n\in \lc 3,5,7,9\rc$ and use $\rho = \sqrt{\frac{\log (p+1)}{mn}}$ for all pairs of $(m,n)$. $T$ is defined as above. For each $j\in[n]$, $X^{(i)}_j\sim 0.5\Ncal(0,\Sigma^{(i)})+0.5\Delta_i$, where $\Delta_i$ is a $p\times 1$ vector of independent $Uniform(0,1)$ random variables. \Cref{fig:big-dim} depicts the outcome of our experiments. For different choices of $n$, the graphs overlap each other perfectly (both $\prob(\hat{J}=J)$ and $\ninf{\hat{\Pi}-\Pi}$). Then, we took the recovered support and derived the probability of support recovery for a novel task. 
\begin{figure}[!ht]
\begin{center}
\begin{tabular}{ccc}
  \includegraphics[width=0.29\textwidth]{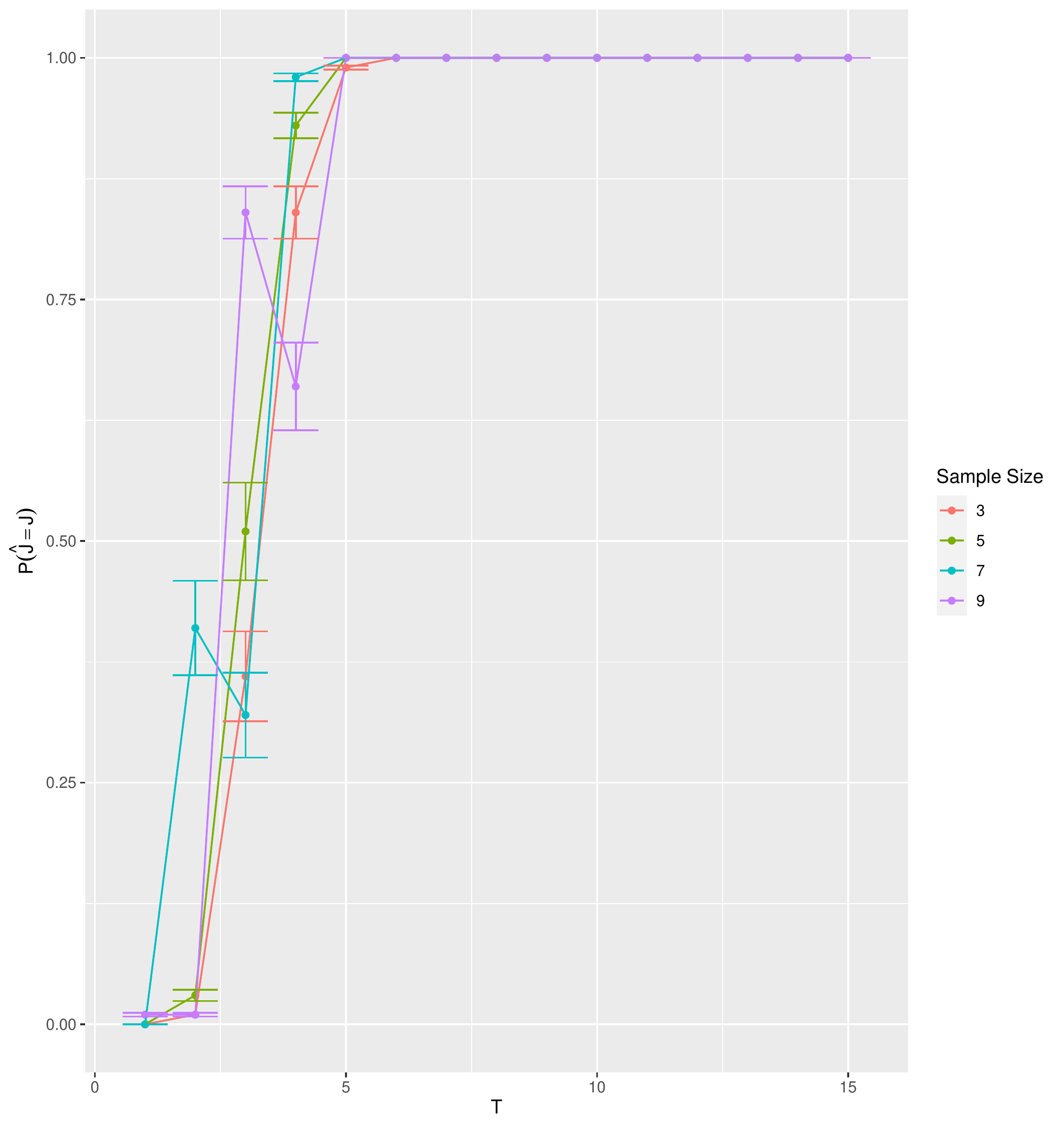} &   \includegraphics[width=0.29\textwidth]{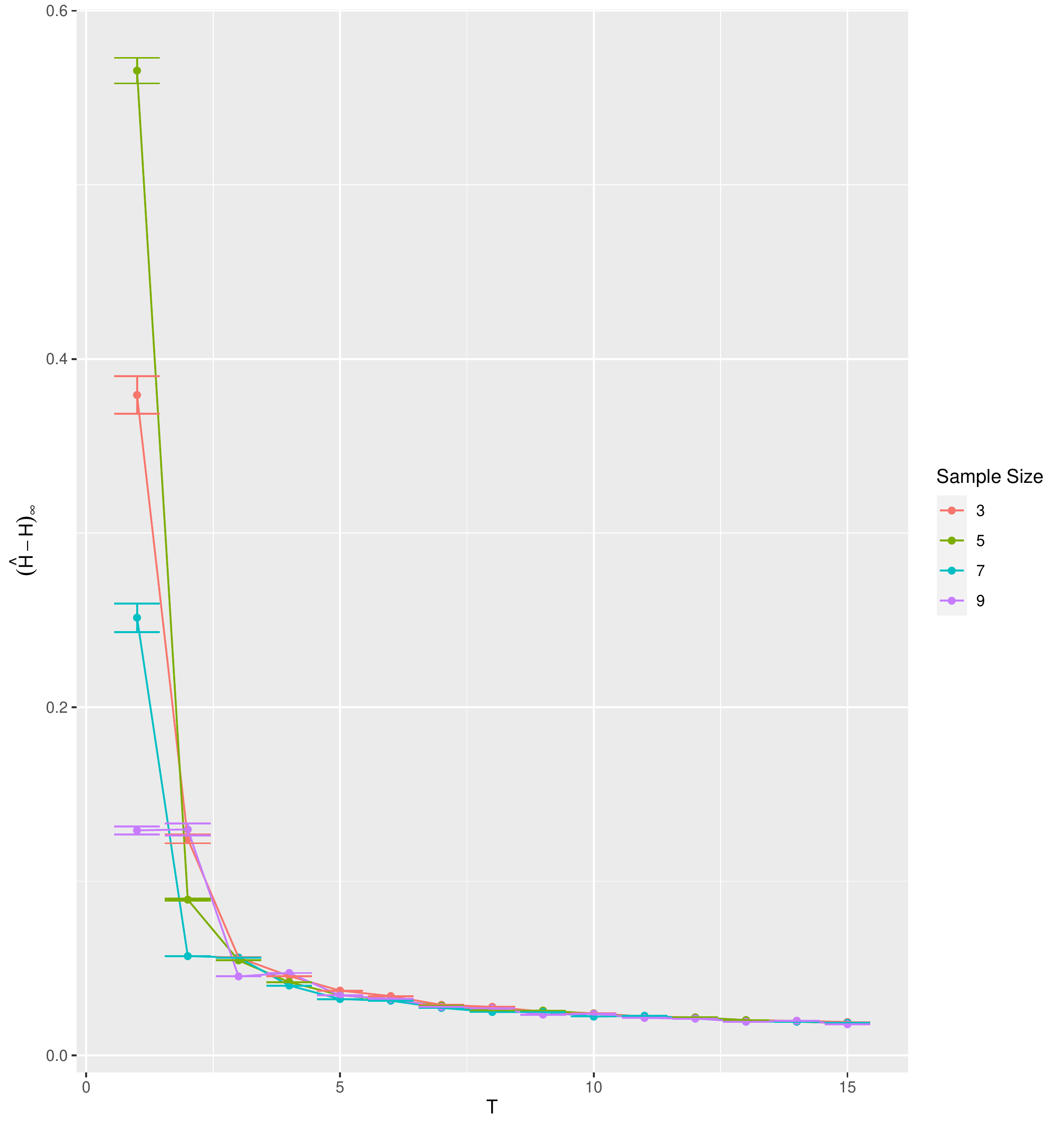}&
  \includegraphics[width=0.37\textwidth]{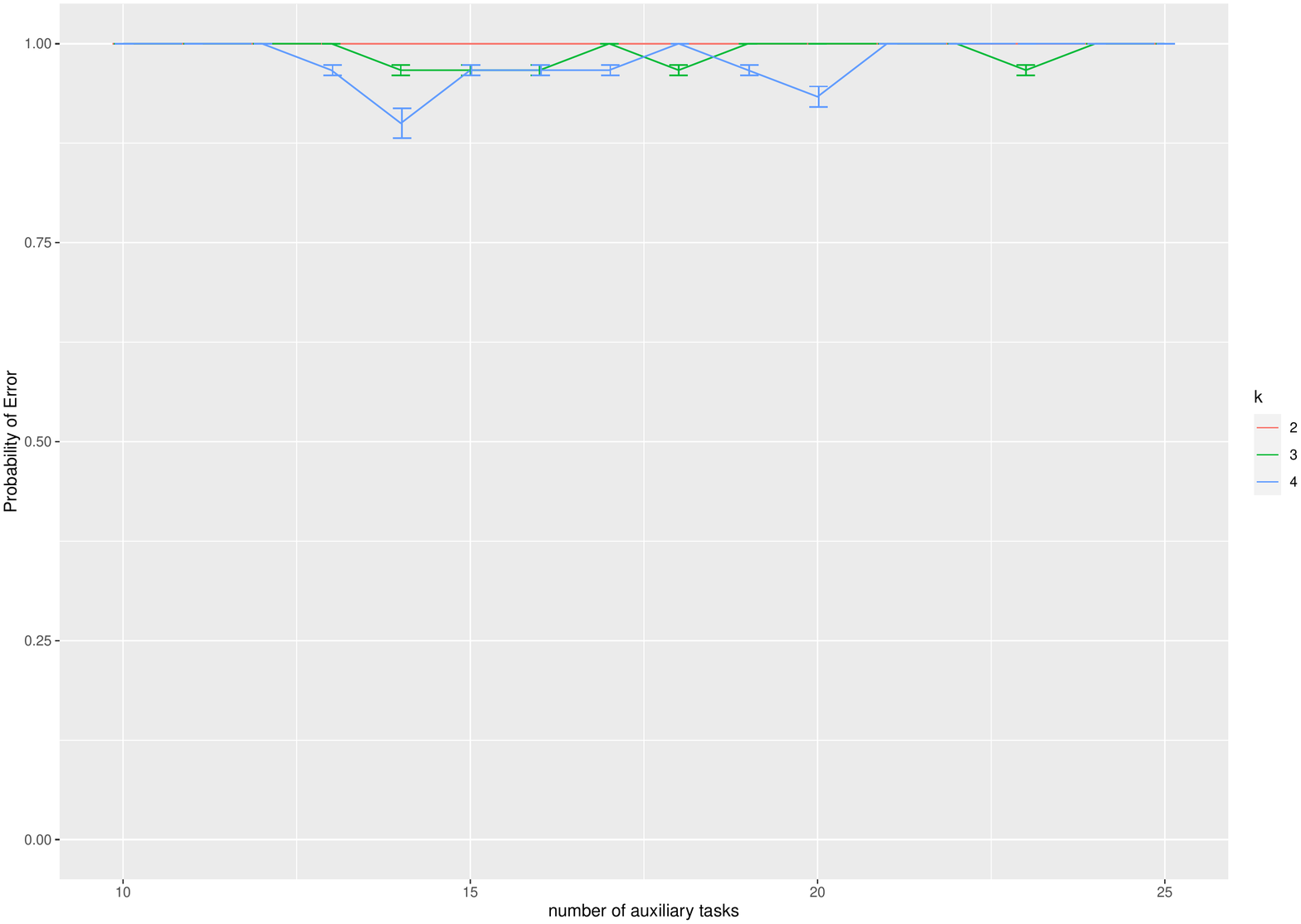}\\
(a) Support Recovery & (b) Maximal Error & (c) Novel task
\end{tabular}
\caption{Left, Center: Simulations for \Cref{thm:meta} on the probability of support union recovery and maximal error under various settings of $n$. Let $\rho = \sqrt{\frac{\log (p+1)}{mn}}$. The x-axis is set by $T:=\frac{mn}{\log(p+1)}$ varying from $\lc1,\dots,15\rc$ for support recovery. Right: Probability of support recovery for novel task. We let $\rho = \sqrt{\frac{\log (|J|+1)}{n}}$. The x-axis is set by $T:=\frac{n}{\log(|J|+1)}$ varying from $\lc10,\dots,25\rc$ }
~\label{fig:big-dim}
\end{center}

We perform the experiments when there are $2,3$, or $4$ extra zeros and in all of those experiments we have found greater than $95\%$ probability of accurately identifying the extra zeros.

\end{figure}

\paragraph{Sub-Gaussian Distribution.} For all the experiments in this sub-section, we let $n=5$, $k=6$ and $|J|=6$ and perform 100 repetitions of each setting. We now consider the setting of different dimensions. We choose $p\in \lc 40,45,50,55\rc$. $T$ and $\rho$ are defined as above. For each $j\in[n]$, $X^{(i)}_j\sim 0.5\Ncal(0,\Sigma^{(i)})+0.5\Delta_i$, where $\Delta_i$ is a $p\times 1$ vector of independent $Exponential(0,1)$ random variables. \Cref{fig:change-dim} depicts the outcome of our experiments. For different choices of $p$, the graphs overlap each other perfectly (both $\prob(\hat{J}=J)$ and $\ninf{\hat{\Pi}-\Pi}$). Then, we took the recovered support union and derived the probability of support recovery for a novel task. 
{We perform the experiments when there are $2,3$, or $4$ extra zeros and in all of those experiments we have found greater than  $90\%$ probability of accurately identifying the extra zeros.}
\begin{figure}[!ht]
    \centering
    \subfloat[Support Union Recovery]{{\includegraphics[width=0.29\textwidth]{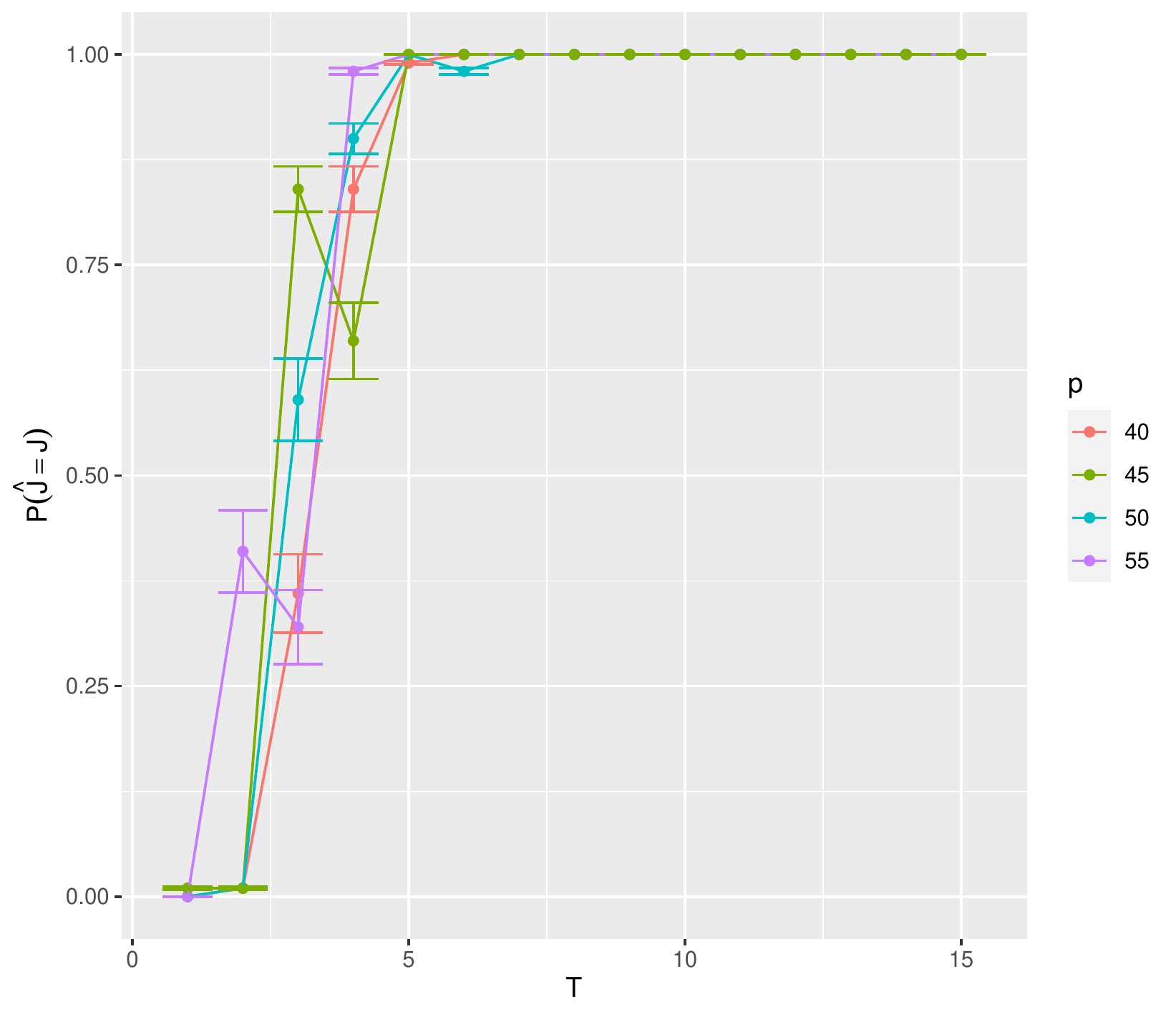}}}
    \subfloat[Maximal Error]{{\includegraphics[width=0.24\textwidth]{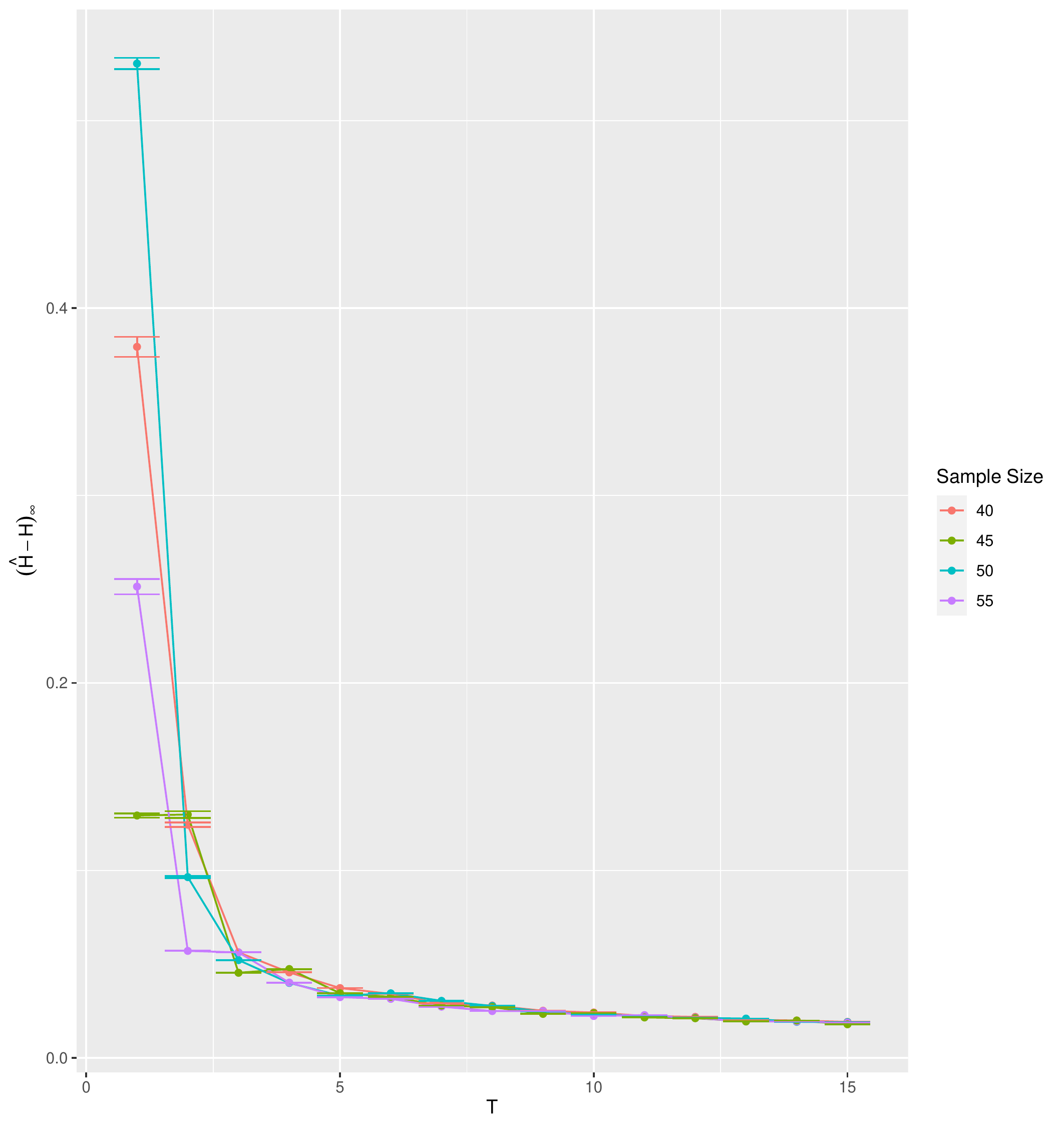}}}
    \subfloat[Novel Task]{\includegraphics[width=0.3\textwidth]{{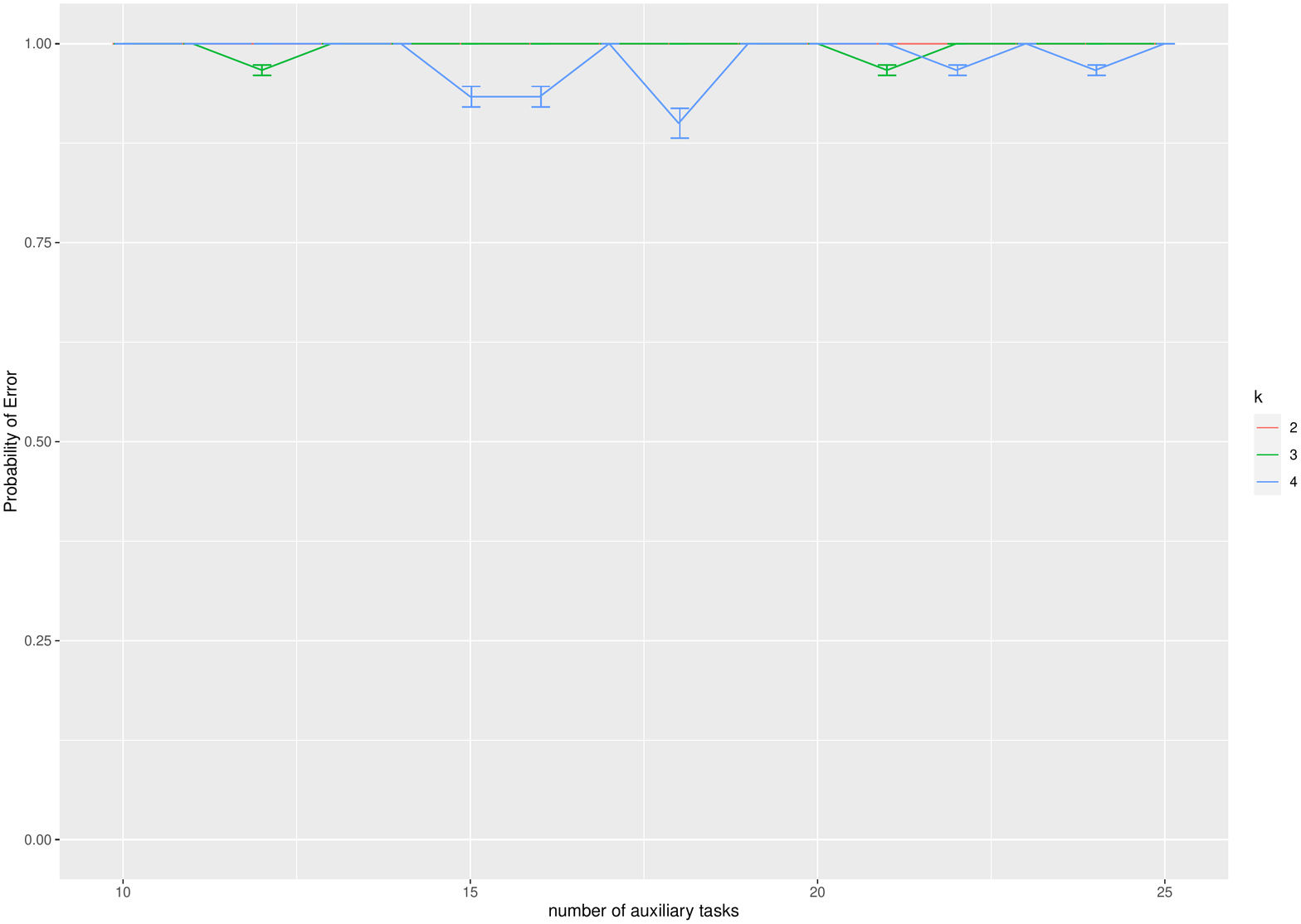}}}
    \caption{Left, Center: Simulations for \Cref{thm:meta} on the probability of support union recovery and maximal error under various settings of $p$. We let $\rho = \sqrt{\frac{\log (p+1)}{mn}}$. The x-axis is set by $T:=\frac{mn}{\log(p+1)}$ varying from $\lc1,\dots,15\rc$ for support union recovery. Right: Probability of support recovery for novel task. We take $\rho = \sqrt{\frac{\log (|J|+1)}{n}}$. The x-axis is set by $T:=\frac{n}{\log(|J|+1)}$ varying from $\lc10,\dots,25\rc$ }\label{fig:change-dim}
    \vspace{-.8cm}
\end{figure}
\section{Real World Datasets}\label{app:real-world}
\paragraph{Brain Imaging Data.}
The 1000 functional connectomes dataset is from \href{http://www.nitrc.org/projects/fcon_1000/}{nitrc.org} and contains resting-state functional magnetic resonance signals from 1128 subjects, 41 sites around the world, and 157 brain regions. We use a set of independent samples to construct the true covariance matrix, and thus the true support union. We then constructed the pooled estimator using $m$ auxiliary tasks, where $m\in\lc1,2,3,5,8,10,12,15,18,20,25,30,35,41\rc$. Then we compared the recovered support with the true support to calculate the mean percentage and standard deviation of error in the recovered support. The error bars were obtained by creating a 95\% confidence interval around the mean. As we clearly see the probability of incurring an error decreases as we increase the number of tasks.

\begin{figure}[!ht]
    \centering
    \includegraphics[scale = 0.3]{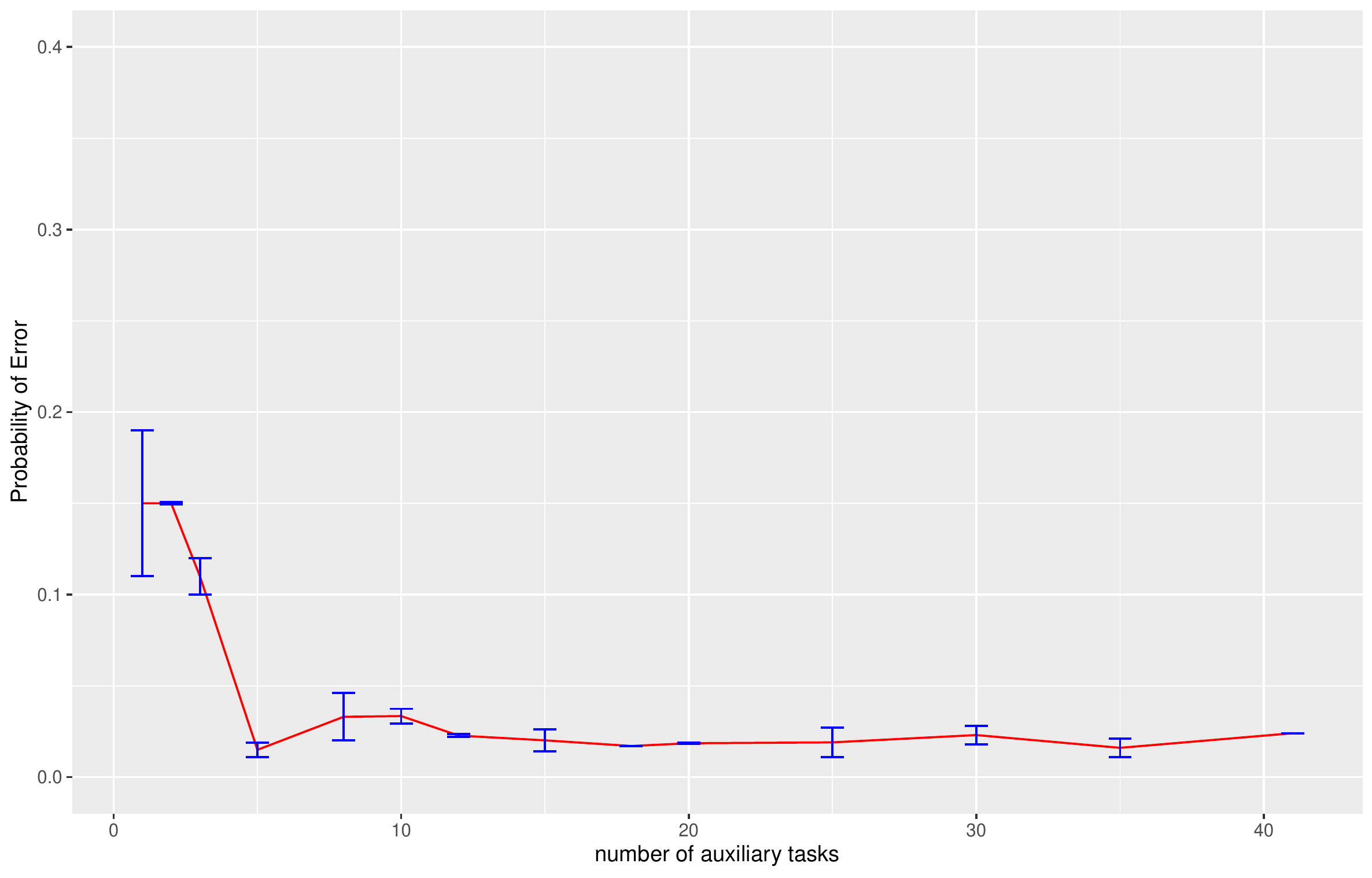}
    \caption{Probability of error of support recovery for brain imaging dataset. We took an increasing number of auxiliary tasks $m\in\lc1,2,3,5,8,\dots,30,35,41\rc$. The probability of error decreases.}
    \label{fig:my_label1}
\end{figure}

\paragraph{Cancer Genetics Data.} The cancer genome atlas program dataset is available at \href{https://portal.gdc.cancer.gov/}{gdc.cancer.gov} and contains the genetic sequencing of 1576 patients with 11 different categories of cancer. We use a set of independent samples to construct the true covariance matrix and the true support union. We then constructed the pooled estimator using $m$ auxiliary tasks where $m\in\{1,\dots 10\}$. Then we compared the recovered support with the true support to calculate the mean percentage and the standard deviation of the error in recovered support. The error bars were obtained by empirically calculating the $95\%$ confidence interval around the mean. As we can see the probability of incurring an error decreases as we increase the number of tasks.

\begin{figure}[!htbp]
    \centering
    \includegraphics[scale = 0.3]{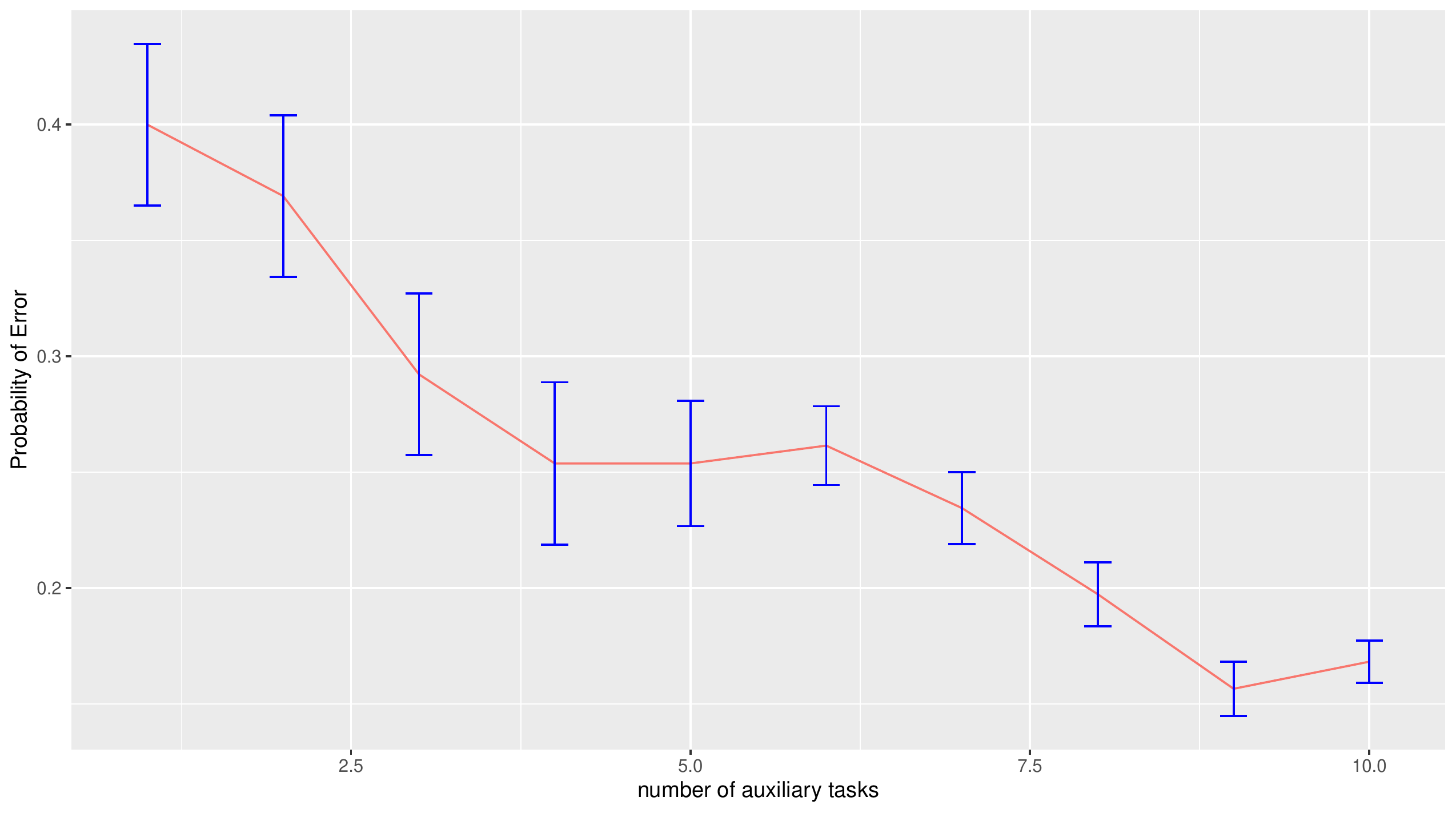}
    \caption{Probability of error of support recovery for cancer genetics dataset. We took an increasing number of auxiliary tasks $m\in\lc1,2,\dots,10\rc$. The probability of error can be seen to be decreasing.}
    \label{fig:my_label1}
\end{figure}

\section{Comparisons}

\subsection{Multiple Independent Sparse PCA. }In this section, we compare the meta-learning method against a multiple independent learning method when applied to support union recovery under a Gaussian distribution setting. 
The method of generating the dataset and meta-learning is as described in \Cref{sec:gauss-err}.
As before, we let $p=50$, $|J|=5$, and fix $n=3$. 
The values we choose for $m$ are chosen from the set  $\mathcal{M}\in\{2,3,4,6,7,8,10,11,12,14\}$. We perform 16 replications and report the mean probability and the 95\% confidence interval as the error bars around the mean. The results are plotted in \cref{fig:comparisons}.

\paragraph{Multiple Independent Learning. }For the multiple independent learning setup there are $m$ many objective functions being maximised simultaneously. Let $S^{(i)}$ be the covariance matrix associated with the task $i$. Then the $i$-th objective is to maximise,

\begin{align}
    \inner{S^{(i)}}{H}-\rho\|H\|_{1,1} \text{ subject to } H\in\Fcal^{k}.
\end{align}
 Let $\hat H^{(i)}$ be the maximiser of this objective function. Let $\hat J^{(i)}=supp(H^{(i)})$ and take $\hat J = \bigcup_{i\in \mathcal{M}} \hat{J}^{(i)}$ to be the natural estimator for the support union. We apply $\sqrt{\frac{\log (p+1)}{n}}$ as the penalty for each of the $m$ tasks.
\paragraph{Meta-Learning. }For the meta-learning setup, we proceed according to \Cref{sec:gauss-err} and recover the support union from the pooled covariance estimator. We apply $\sqrt{\frac{\log (p+1)}{mn}}$ as the penalty when $m$ is the total number of tasks.

Multiple independent learning of sparse PCA can be seen to be completely useless whenever the number of tasks get large and fails to recover the correct support union after a while. On the other hand, meta-learning improves steadily with the increasing number of tasks and the results coincide with our experiments in \Cref{sec:gauss-err}.
\begin{figure}[!h]
    \centering
    \includegraphics[width=0.5\textwidth]{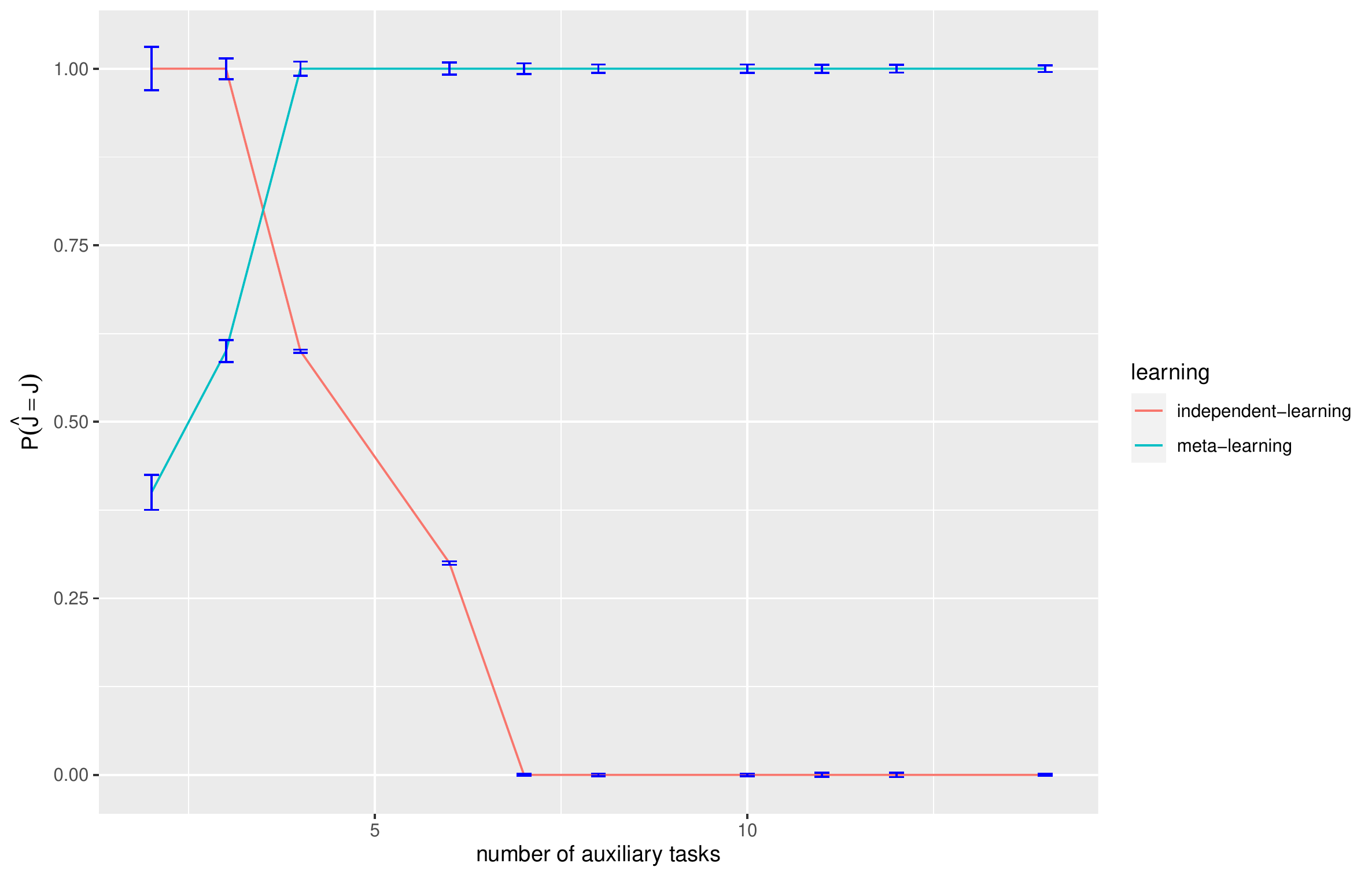}
    \caption{Comparison showing the difference in the probability of recovering the correct support union of meta-learning versus multiple independent sparse PCA. y-axis : Probability of support union recovery. x-axis : Number of auxiliary tasks. }
    \label{fig:comparisons}
\end{figure}

\subsection{Comparison With Multi-Task PCA. }In this section, we compare the meta-learning method against a multi-task learning method when applied to support union recovery under a Gaussian distribution setting. 
The method of generating the dataset and meta-learning is as described in \Cref{sec:gauss-err}.
As before, we let $p=50$, $|J|=5$, and fix $n=3$. 
The values we choose for $m$ are chosen from the set  $\mathcal{M}\in\{2,3,4,6,7,8,10,11,12,14\}$. We perform 16 replications and report the mean probability and the 95\% confidence interval as the error bars around the mean. The results are plotted in \cref{fig:comparisons}.
\begin{figure}[!ht]
    \centering
    \includegraphics[width=0.5\textwidth]{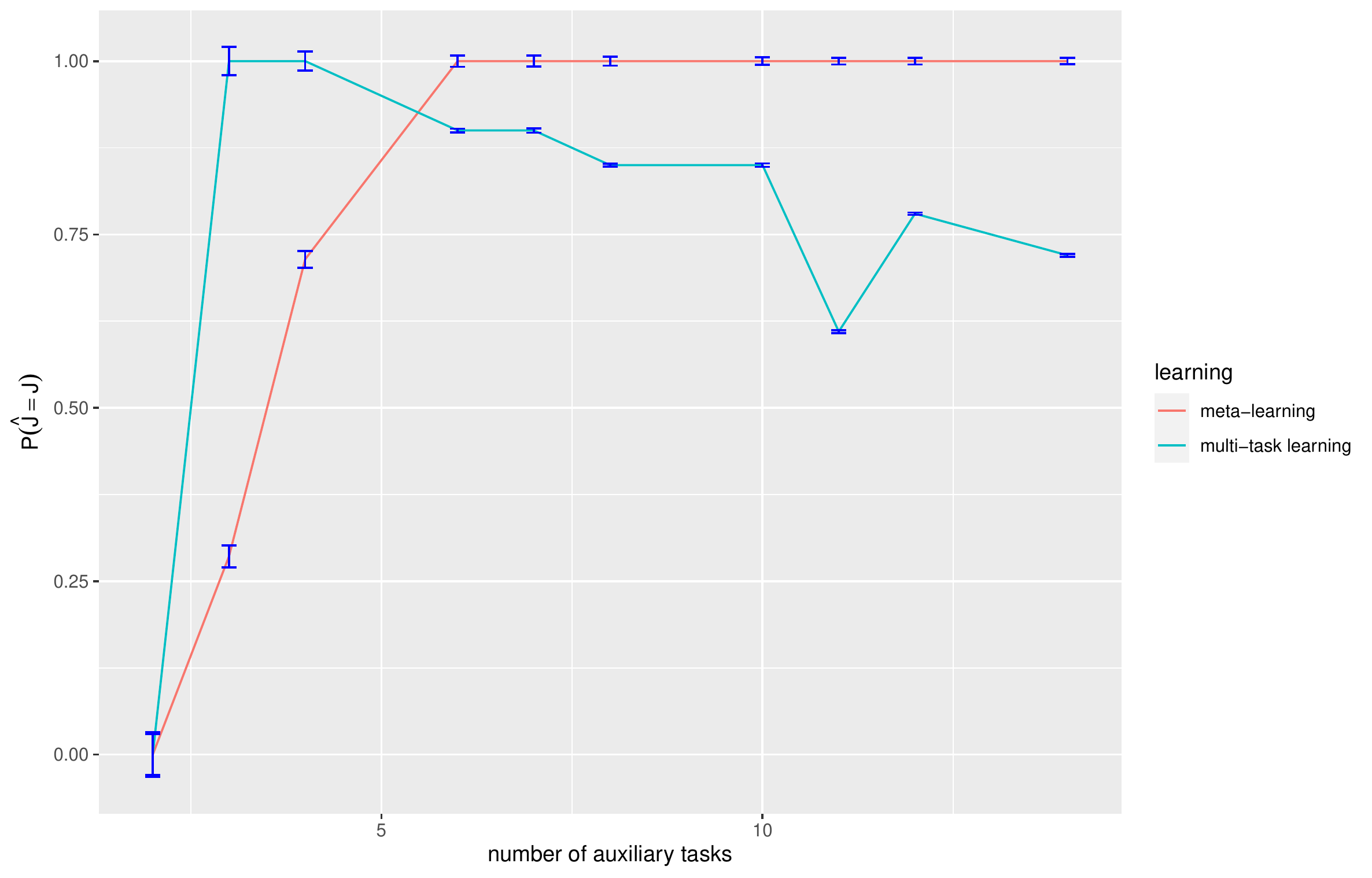}
    \caption{Comparison showing the difference in the probability of recovering the correct support union of meta-learning versus multi-task learning. y-axis : Probability of support union recovery. x-axis : Number of auxiliary tasks. }
    \label{fig:comparisons}
\end{figure}
\paragraph{Multi-Task Learning. }For the multi-task learning method, we maximise the following objective function:
\begin{align*}
    \sum_{i=1}^m \inner{S^{(i)}}{H^{(i)}}-\rho\|H^{(1)},\dots,H^{(m)}\|_1
    \text{ subject to $H^{(i)}\in\Fcal^k\ \forall \ i\in{1,\dots,m}$ },
\end{align*}
where $S^{(i)}$ is the covariance matrix corresponding to task $i$ and $\rho\|H^{(1)},\dots,H^{(m)}\|_1=\sum_{j,k} \sup_{i\in\{1,\dots,m\}} |H_{j,k}^{(i)}|$. We apply $\sqrt{\frac{\log (p+1)}{mn}}$ as the penalty for the tasks. We recover the support union by taking the union of the recovered support.

\paragraph{Meta-Learning. }For the meta-learning setup, proceed according to \Cref{sec:gauss-err} and recover the support union from the pooled covariance estimator. We apply $\sqrt{\frac{\log (p+1)}{mn}}$ as the penalty when $m$ is the total number of tasks.

As we can see, meta-learning outperforms multi-task learning with 100\% accuracy for large values of $m$.

\end{document}